\documentclass[letterpaper, 10 pt, conference]{ieeeconf}  %

\IEEEoverridecommandlockouts                              
\usepackage{times}
\usepackage{multicol}
\usepackage{graphicx}
\usepackage{subfigure}
\usepackage{epsfig}
\usepackage{amsmath,amssymb,bm}
\usepackage{algorithmic}
\usepackage{url}
\usepackage{color}
\usepackage[ruled]{algorithm2e}
\usepackage{ntheorem}
\usepackage{float}
\usepackage{physics}
\usepackage{cite}
\usepackage{todonotes}
\usepackage{mathtools}

\usepackage[normalem]{ulem}
\usepackage{svg}
\usepackage{soul}

\newcommand{\mb}[1]{{\mathbf{#1}}}
\newcommand{\mc}[1]{{\mathcal{#1}}}

\newcommand{\revo}[1]{{\color{black}{#1}}}

\newcommand{\PL}[1]{{\color{red}{#1}}}

\newtheorem{thm}{Theorem}

\newtheorem{rem}{Remark}
\newtheorem{problem}{Problem}

\usepackage{amsmath}

\title{\LARGE \bf
Assignment Algorithms for Multi-Robot  Multi-Target Tracking \\ with Sufficient and Limited Sensing Capability 
}
\author{Peihan Li and Lifeng Zhou % <-this % stops a space
\thanks{The authors are with the Department of Electrical and Computer Engineering, Drexel University, Philadelphia, PA 19104, USA. Email: \texttt{\small \{pl525,lz457\}@drexel.edu}.}
\thanks{This research is sponsored by the Army Research Lab through ARL DCIST CRA W911NF-17-2-0181.}}
\begin{document}

\maketitle
\thispagestyle{empty}
\pagestyle{empty}

\begin{abstract}
We study the problem of assigning robots with actions to track targets. The objective is to optimize the robot team's tracking quality which can be defined as the reduction in the uncertainty of the targets' states. Specifically, we consider two assignment problems given the different sensing capabilities of the robots. In the first assignment problem, a single robot is sufficient to track a target. To this end, we present a greedy algorithm (Algorithm~\ref{algorithm:complete_assignment}) that assigns a robot with its action to each target. We prove that the greedy algorithm has a $1/2$--approximation bound and runs in polynomial time. Then, we study the second assignment problem where two robots are necessary to track a target. We design another greedy algorithm (Algorithm~\ref{algorithm:limited_assignment}) that assigns a pair of robots with their actions to each target. We prove that the greedy algorithm achieves a $1/3$--approximation bound and has a polynomial running time. Moreover, we illustrate the performance of the two greedy algorithms in the ROS-Gazebo environment where the tracking patterns of one robot following one target using Algorithm~\ref{algorithm:complete_assignment} and two robots following one target using Algorithm~\ref{algorithm:limited_assignment} are clearly observed. Further, we conduct extensive comparisons to demonstrate that the two greedy algorithms perform close to their optimal counterparts and much better than their respective ($1/2$ and $1/3$) approximation bounds.

\end{abstract}

\section{Introduction}
Tracking and localizing targets are essential for multi-robot systems to perform tasks such as surveillance, patrolling, and monitoring~\cite{grocholsky2006cooperative,tokekar2013tracking,robin2016multi}. One line of research is to formulate the target tracking as an assignment problem where robots are appropriately assigned to targets to optimize the tracking quality~\cite{jiang2003optimal,spletzer2003dynamic,kamath2007triangulation,zavlanos2008dynamic,tekdas2010sensor,nam2015your,chopra2017distributed,zhou2019sensor,yang2020algorithm}. Particularly, the problem of assigning one robot to track one target (i.e., one-to-one assignment) can be formulated as a bipartite graph matching problem, which can be solved using the Hungarian algorithm~\cite{kuhn1955hungarian}. For the many-to-one assignment problem (i.e., assigning multiple robots to one target), approximation algorithms have been designed to reach the near-optimal tracking performance in our previous work~\cite{zhou2019sensor}. 

% Multiple variants of the assignment problem has been studied such as one-to-one assignment (i.e., assigning one robot to one target), one-to-many assignment (i.e., assigning one robot to multiple targets), and 
% many-to-one assignment (i.e., assigning multiple robots to one target)~\cite{}. The one-to-one assignment is typically formulated as a bipartite graph matching problem and can be solved using the Hungarian algorithm~\cite{kuhn1955hungarian}. The one-to-many assignment can be formulated as a mix-integer 

% Furthermore, tracking moving targets can be more challenging due to the fact that robots usually have limited knowledge about the target's motion model\cite{spletzer2003dynamic,atanasov2014information,zhou2011multirobot,tokekar2014multi,zhou2018active,zhou2019sensor}. \LZ{assignment problems... }

Instead of assigning robots directly, we study the problem of assigning robots with actions to track targets (Fig.~\ref{fig:assignment}). Specifically, we consider that each robot has a set of candidate actions (or control inputs), and it chooses one action to execute at each step (due to the natural constraint)~\cite{schlotfeldt2018anytime}. Assigning robot-actions\footnote{For convenience, we henceforth use ``robot-action'' to represent ``robot with action".}  addresses the problems of robot-target assignment and robot-action selection concurrently and enables the robots to actively move and track the assigned targets by executing actions. In addition, similar to \cite{zhou2019sensor}, we consider each robot-action to be assigned to at most one target but allow multiple robot-actions to be assigned to one target. This is motivated by the case where measuring multiple targets by a single robot is time-consuming or even impossible, as is the case for radio sensors~\cite{tokekar2011active,alvarez2017acoustic}. 
% In this paper, we study assignment problem for a multi-robot group tracking multiple active targets based on assigning actions for robots that would improve the tracking quality with noisy measurements. 
With these settings, we study two assignment problems given the different sensing capabilities of the robots and provide a constant-factor approximation algorithm for each. \PL{}

    \begin{figure}[tbh]
    \centering{
    \subfigure[Assignment with sufficient sensing capability]{
    \includegraphics[width=0.479\columnwidth]{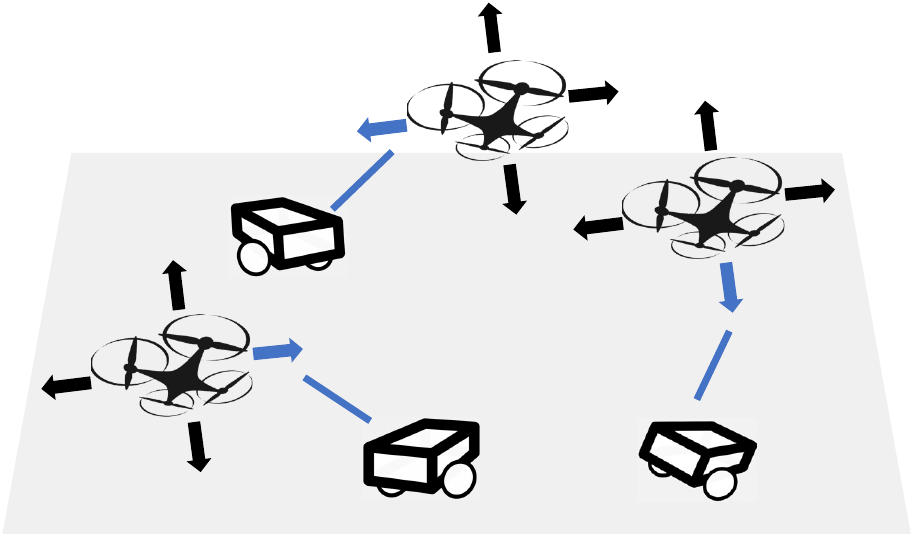}}
    \subfigure[Assignment with limited sensing capability]{
    \includegraphics[width=0.479\columnwidth]{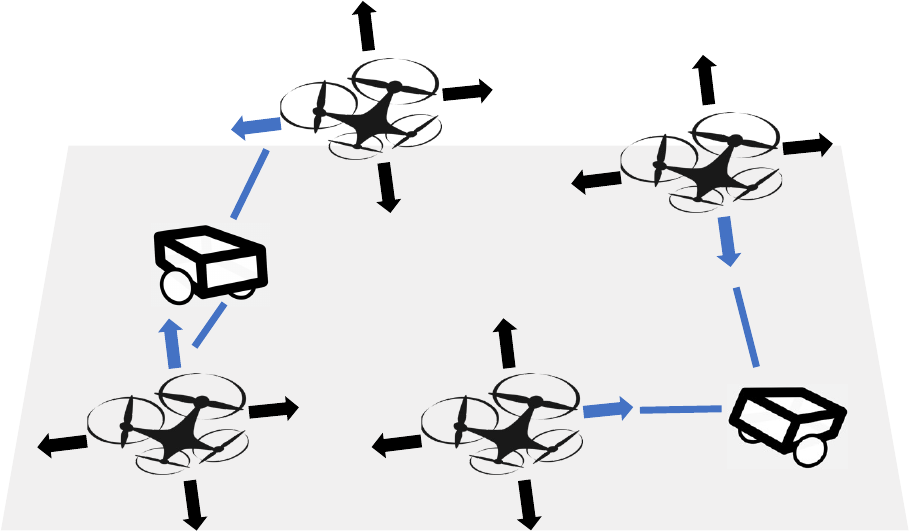}}
    }
    \caption{Assigning robots (drones) with actions to track targets (rovers). Each robot has four candidate actions (represented by the arrows), and it chooses one action (the blue arrow) to execute at a step. The solid blue line represents the robot-action to target assignment. (a): Assigning one robot with action to one target. (b): Assigning a pair of robots with actions to one target. \label{fig:assignment}}
    \end{figure}

We name the first problem as \textit{assignment with sufficient sensing} (Problem~\ref{prob:problem1}, as shown in Fig.~\ref{fig:assignment}-(a)). In this problem, we consider each robot to be sufficient to estimate the state of a target, e.g., a robot with a range-and-bearing sensor. To tackle the problem, we design a greedy algorithm that evaluates the tracking quality of assigning each robot-action to each target and picks the pair of robot-action and target that gives the maximum tracking quality at each round (Algorithm~\ref{algorithm:complete_assignment}). We then prove that the greedy algorithm has a $1/2$--approximation bound and runs in polynomial time. 

% We consider two scenarios here. One is that each robot is capable of estimating the target position by itself, which we denote as \textit{sufficient sensing capability}. 
% We first investigate the scenario with sufficient sensing capability, where one robot is sufficient for tracking a target (e.g., range and bearing sensor). Previous work \cite{zhou2019sensor} has shown that with greedy algorithm assigning one stationary sensor to track one target can provide $1/2$-approximation of the optimal solution. We propose a new greedy algorithm for the action selection. We will then proof our proposed greedy algorithm (Algorithm \ref{algorithm:complete_assignment}) can still provide $1/2$-approximation of the optimal solution with one robot tracking one target scenario.

We then study the assignment problem where two robots are necessary to estimate the state of a target, which we name as \textit{assignment with limited sensing} (Problem~\ref{prob:problem2}, as shown in Fig.~\ref{fig:assignment}-(b)). This typically happens when the robot uses range or bearing sensor only to measure the state of the target, and thus at least two robots are required to track a target~\cite{kamath2007triangulation, tekdas2010sensor,zhou2019sensor}. To address the problem, we devise a greedy algorithm that evaluates the tracking quality of assigning each pair of robot-actions to each target and picks the triple of robot-actions and target with the maximum tracking quality at each round (Algorithm~\ref{algorithm:limited_assignment}). We also prove that the greedy algorithm has a $1/3$--approximation ratio and a  polynomial running time. Notably, our previous work~\cite{zhou2019sensor} has presented a greedy algorithm that assigns sensors to optimize observability while tracking targets. Differently, here, we focus on assigning robot-actions that allow robots to actively follow and track the targets. 

% An example of such assignment strategy includes optimizing the tracking quality with two bearing sensors\cite{kamath2007triangulation, tekdas2010sensor}. 
% As the fact that observing multiple targets by a single robot is time-consuming or impractical, one motivation for our formulation is to decrease the load for each robot in completing the tracking task. Therefore, we set the constraint in the assignment that each robot can only observe one target at any time, but multiple robots can observe a target collaboratively.

% Then, we study the scenario with limited sensing capability, where two robots are required to track a target. In our case, we consider the robots use only the range sensor, and two robots need to observe the target simultaneously to get the target position estimation. In the stationary sensors case, the greedy algorithm would give a $1/3$-approximation of the optimal solution \cite{zhou2019sensor}. We propose a new greedy algorithm for action selection, and also provide $1/3$-approximation with proof.

\textbf{Contributions.} We make four key contributions. 
\begin{itemize}
    \item (\textit{Problem}) We formulate the problems of assigning robot-actions to track targets with both sufficient and limited sensing capability to optimize the team's tracking quality. 
    \item (\textit{Solution}) We propose two greedy algorithms to address the assignment problems, i.e., Algorithm \ref{algorithm:complete_assignment} for Problem~\ref{algorithm:complete_assignment} and Algorithm \ref{algorithm:limited_assignment} for Problem~\ref{algorithm:limited_assignment}.
    \item (\textit{Analysis}) We derive and prove a $1/2$--approximation bound for Algorithm \ref{algorithm:complete_assignment} and a $1/3$--approximation bound for Algorithm \ref{algorithm:limited_assignment}. Moreover, we prove that Algorithm \ref{algorithm:complete_assignment} and Algorithm \ref{algorithm:limited_assignment} run in polynomial time.  

    \item (\textit{Evaluation}) We evaluate Algorithm \ref{algorithm:complete_assignment} and Algorithm \ref{algorithm:limited_assignment} both qualitatively and quantitatively. The qualitative results show that both algorithms consistently steer the robots toward following and tracking the targets. The quantitative results demonstrate that the algorithms perform close to the optimal solution and much better than the theoretical approximation bound. 
\end{itemize}

\section{Problem Formulation}
In this section, we formalize the problem of assigning robots with their actions to track multiple targets. To start with, we clarify the notations used in this work. 

\textbf{Notation.} We denote uppercase letters as scalars (e.g., $X$), bold lowercase letters as vectors (e.g., $\mb{x}$), and bold uppercase letters as matrices (e.g., $\mb{X}$). 
% $\texttt{Tr}(\mb{X})$ denotes the trace of matrix $\mb{X}$. 
The calligraphic letter $\mc{S}$ denotes a set. $|\mc{S}|$ denotes the cardinality of set $\mc{S}$. 

Next, we formalize the framework's convention and define the problem.
\subsection{Framework} \label{subsec: framework}
\paragraph{Robot}
We consider a target tracking scenario with a team of $N$ mobile robots, denoted as $\mathcal{R}=\{1, \cdots, N\}$. The robot team is tasked to track multiple moving targets using onboard sensors, and each robot $i\in \mathcal{R}$ has the discrete-time motion model: 
    \begin{equation}
        \mb{x}_{i,t+1}  =  f_i(\mb{x}_{i,t}, \mb{a}_{i,t}), ~\forall i \in \mathcal{R},
        \label{eq:robot_motion}
    \end{equation}
where $\mb{x}_i$ denotes the state of robot $i$ and $\mb{a}_i \in \mc{A}_i$ denotes its action (or control input). $\mc{A}_i$ is a finite set of candidate actions that the robot $i$ can choose from. $\mc{A}:=\cup_{i=1}^N \mc{A}_i$ is the joint action set of all robots. Notably, a robot can only select one action from its candidate action set to execute at each time step, which is a natural constraint. 
% Furthermore, $\mb{u}_{\mc{R}} = \{\mb{u}_1, \mb{u}_2, \cdots, \mb{u}_N\}$ denotes the set of actions chosen for the robot team. 

\paragraph{Target}
We consider $M$ targets to be tracked in the scenario, denoted as $\mc{T}=\{1, \cdots, M\}$. Each target $j\in \mc{T}$ has the discrete-time stochastic motion model:
    \begin{equation}
        \mb{y}_{j,t+1}  =  g_j(\mb{y}_{j,t}) + \mb{w}_{j,t}, ~\forall j \in \mathcal{T},
        \label{eq:target_motion}
    \end{equation}
where $\mb{y}_j$ denotes the state of target $j$, and $\mb{w}_{j,t}$ denotes the zero mean white Gaussian process noise with covariance $\mb{Q}_{j,t}$, i.e., $\mb{w}_{j,t} \sim \mathcal{N}(0, \mb{Q}_{j,t})$. We assume $\mb{w}_{j,t}$ to be independent of the process noises of other targets.

\paragraph{Sensing}
\label{para: sensing}
We consider that a robot $i \in \mc{R}$ observes a target $j \in \mc{T}$ by the following sensing model:
    \begin{equation}
        \mb{z}_{i,t}^j = h_i^j(\mb{x}_{i,t}, \mb{y}_{j,t}) + \mb{v}_{i,t}^j(\mb{x}_{i,t}, \mb{y}_{j,t}), ~i\in \mc{R}, j\in \mc{T},
    \label{eq:measure_model}
    \end{equation}
where $\mb{z}_{i,t}^j$ denotes the measurement of target $j$ by robot $i$'s sensor at time $t$. $\mb{v}_{i,t}^j(\mb{x}_{i,t}, \mb{y}_{j,t})$ denotes the Gaussian measurement noise with zero mean and covariance $\mb{R}(\mb{x}_{i,t}, \mb{y}_{j,t})$, i.e., $\mb{v}_{i,t}^j(\mb{x}_{i,t}, \mb{y}_{j,t}) \sim \mc{N}(0, \mb{R}(\mb{x}_{i,t}, \mb{y}_{j,t}))$. Here, both the sensing model $h_i^j(\mb{x}_{i,t}, \mb{y}_{j,t})$ and measurement noise $\mb{v}_{i,t}^j(\mb{x}_{i,t}, \mb{y}_{j,t})$ depend on the state of robot $i$ and target $j$. 

\paragraph{Objective}
\label{para: objective}
The objective is to maximize the tracking quality of the targets by assigning robots and their selected actions to the targets. This differs from our previous work\cite{zhou2019sensor} where we assign stationary sensors to localize targets. Here, considering that robots are moving by executing actions, we assign robot-actions (per robot an action) to track the targets. That way, we denote $\phi(j)$ as the set of robot-actions
% \footnote{For simplicity, we henceforth consider ``assigning action'' and ``assigning robot-action" interchangeably since each (selected) action is attached to a specific robot.} 
assigned to target $j$ and $\phi^{-1}(i^k)$ as the set of targets assigned to robot-action $i^k$ (i.e., robot $i$ with action $k$). 
% Denote the robot that action $k$ belongs to as $i(k)$.
We use $\phi_l(j)$ to denote the $l^\text{th}$ robot-action assigned to target $j$. We order the assigned robot-actions based on the robots' IDs such that $\phi_1(j) \leq \phi_2(j) \leq \cdots \leq \phi_L(j)$. Denote the corresponding robot for  each robot-action $\phi_l(j)$ as $i(\phi_l(j))$. 

%by using their corresponding robots' IDs
% Here, if robot $i$ selects action $\mb{u}_i \in \mc{V}$, we denote this robot and action combination as a robot-action pair $i_{\mb{u}_\mc{V}}$. The set of all the candidate robot-action pair is denoted as $\mc{R}_\mc{U}$, and the set of selected robot-action pairs is denoted as $\mc{R}_\mc{V}$ Due to the natural constraint that each robot can only select one action at any time step, the set of robot-action pair has the same cardinality as the robot set. 

The target state is estimated using the extended Kalman filter (EKF) with the measurements taken by the robots. Therefore, the tracking quality, denoted by $q$, can be computed using the covariance matrix from EKF~\cite[Sec. 4.1]{jawaid2015submodularity}. Specifically, we can define the tracking quality by metrics such as the trace, log determinant, or maximum eigenvalue of the covariance matrix~\cite{jawaid2015submodularity}. As an example, the tracking quality $q$ can be defined as the difference between the trace of the predicted (\textit{a priori}) covariance and that of  the updated (\textit{a posteriori}) covariance matrix~\cite[Sec. V]{zhou2023robust}. Particularly, we use $q(i^k, j)$ and $q(\{i^k, {i'}^{k'}\}, j)$ to denote the tracking quality of assigning robot-action $i^k$ and robot-actions $\{i^k, {i'}^{k'}\}$ to target $j$, respectively. 

% Particularly, we use $q(k, j)$ and $q(\{k, k'\}, j)$ to denote the tracking quality of assigning action $k$ and actions $\{k, k'\}$ to target $j$, respectively. 
% and $q(\mc{S}, j)$ to denote the tracking quality of assigning action $k$ and a set of actions $\mc{S}$ to target $j$, respectively.

% Similarly, $q(\mc{R}_i^k, j)$ denote the tracking quality of assigning robots and their actions $(i,k)$ to target $j$
% In this paper, we use the trace  $\texttt{Tr}(\mb{\phi}_{\mc{T}})$ as the metric. However, our framework also holds for others metrics.  

 % Each robot $i \in \mc{R}$ selects an action $\mb{u}_i \in \mc{U}_i$ at the current time step $t$ that maximizes the overall tracking quality for the successive time step $t+1$.

\subsection{Problem Definition}
Given the framework introduced in Sec. \ref{subsec: framework}, we formally define the problems of assigning robot-actions to track the targets. We consider two assignment settings---the robot with sufficient and limited sensing capability. The sufficient sensing capability means that a single robot is able to localize and track the target, \textit{e.g.,} a robot has a range-and-bearing sensor. Notably, a range-and-bearing sensor can take both the measurements of the distance and orientation of the target and thus is sufficient to estimate the state of the target~\cite{atanasov2015decentralized,schlotfeldt2018anytime}. In contrast, a robot with limited sensing capability cannot estimate the state of the target by itself. For example, with a range or bearing sensor only, at least two robots are required to track a target~\cite{tekdas2010sensor,zhou2019sensor}. 

For both settings, we constrain each robot-action to
be assigned to one target only. This is motivated by the cases where sensing multiple targets by a single robot can be time-consuming if the robot uses a radio sensor for tracking~\cite{tokekar2011active,alvarez2017acoustic} or communicating multiple measurements can be time and energy-consuming. Therefore, for the robot with sufficient sensing capability, we consider the minimum size of the robot team to be equal to the size of the targets, e.g., $N \geq M$. While for the limited sensing setting, we consider the size of the robot team to be at least 
twice as many 
% two times more than that 
of the targets since a minimum of two robots are required to track one target, e.g., $N \geq 2M$. The two assignment problems are formally defined in Problem~\ref{prob:problem1} and Problem~\ref{prob:problem2}.

% \begin{equation}
%     \Delta \texttt{Tr}(\mb{\phi}_{\mc{T}}) =\texttt{Tr}(\mb{\phi}_{\mc{T}, t-1|t}) - \texttt{Tr}(\mb{\phi}_{\mc{T}, t|t})
% \end{equation}
% where $\mb{\phi}_{\mc{T}, t-1|t}$ is the predicted (\textit{a priori}) covariance and $\mb{\phi}_{\mc{T}, t|t}$ is the updated (\textit{a posteriori}) covariance matrix.

\begin{problem}[Assignment with Sufficient Sensing]
    % Based on the frameworks introduced in Sec. \ref{subsec: framework}, each individual robot in the robot team has the complete sensing capability (e.g. range and bearing sensor). In other words, every single robot can localize the target by itself with the observation taken from the sensor onboard. In our setting, we measure the tracking quality by the reduction in trace of the covariance matrix, which is denoted as follows:
    
    Given a set of robots $\mathcal{R}=\{1, \cdots, N\}$, each with sufficient sensing capability, a set of candidate actions for all robots $\mc{A}$, and a set of targets $\mc{T}=\{1, \cdots, M\}$, find an assignment of robot-actions to targets to
    \begin{align}
    \begin{split}
        \max \sum_{j=1}^M q(\phi_1(j), j)
        % , i \in \mc{R}, j \in \mc{T};
        % &\\
        % |\phi^{-1}(k)| \leq 1&
    \end{split}
    \label{eq:problem1}
    \end{align}
    % and 
    % \begin{equation}
    %     \Delta \emph{\texttt{Tr}}(\Phi_{j, t+1}) = -(\emph{\texttt{Tr}}(\Phi_{j,t+1}) - Tr(\Phi_{j,t}))
    % \label{eq:delta_trace}
    % \end{equation}
    % \begin{equation}
    %         \arg \max_{\mb{u}_i\in \mc{U}_i} \sum_{j=1}^M \Delta Tr(\Phi_{j, t+1}), i \in \mc{R}, j \in \mc{T}, |\mb{u}_i|=1
    %     \end{equation}
    %\textcolor{red}{$|\mb{u}_i|=1$ requires that each individual robot can only have one set of control inputs (e.g., $\mb{u}_i = \{v, w\}$ where $v$ and $w$ are linear and angular velocity respectively) selected in any given time step.} \textcolor{blue}{Does $|\mb{u}_i|=1$ also mean the magnitude of the velocity is 1?} 
    with added constraints that each robot-action is assigned to at most one target, i.e., $|\phi^{-1}(i^k)| \leq 1$, assuming $N\geq M$ and each robot can only select one action at each step. 
    \label{prob:problem1}
\end{problem}

% We consider two scenarios in problem \ref{prob1}, which are the robots with complete sensing capability and robots with limited sensing capability.
% \paragraph{Complete Sensing Capability} \label{scene:complete}

% \LZ{make it Problem 1.1}
%  The covariance matrix in this scenario will be $\Phi_t(\phi(j), t_j)$, where $|\phi(j)|=1$.

% \paragraph{Limited Sensing Capability}  \label{scene:limited}
% \LZ{make it Problem 1.2}

% The covariance matrix in this scenario will be $q_t(\phi_1(j), \phi_2(j), t_j)$, where $|\phi(j)|=2$.
\begin{problem}
    [Assignment with Limited Sensing]
    Given a set of robots $\mathcal{R}=\{1, \cdots, N\}$, each with limited sensing capability, a set of candidate actions for all robots $\mc{A}$, and a set of targets $\mc{T}=\{1, \cdots, M\}$, find an assignment of pairs of robot-actions to targets to
    \begin{align}
    \begin{split}
        \max \sum_{j=1}^M q(\{\phi_1(j), \phi_2(j)\}, j)
        % , i \in \mc{R}, j \in \mc{T};
        % &\\
        % |\phi^{-1}(k)| \leq 1&
    \end{split}
    \label{eq:problem2}
    \end{align}
    with added constraints that each robot-action is assigned to at most one target, i.e., $|\phi^{-1}(i^k)| \leq 1$, assuming $N\geq 2M$ and each robot can only select one action at each step.  
    % where $|\phi^{-1}(i)| = 1$ requires each robot-action pair can only be assigned to one target, and $|\phi(j)|=2$ for each target $j \in \mc{T}$ due to the limited sensing capability. 
    \label{prob:problem2}
\end{problem}
In the next section, we present the assignment algorithms for solving Problem~\ref{prob:problem1} and Problem~\ref{prob:problem2}. 

\section{Assignment Algorithms and Analysis}\label{sec:algana}
% For our problem, we assume that the ground-truth state of the target at time $t$ is known. However, this information is estimated using EKF and the observations from the robots, and we can get the estimated state with a covariance matrix. We use the estimated state of the target for our algorithm. Such practice of getting assignments to better estimate the state with the estimated state to begin with is called cyclical problem \cite{zhou2019sensor}.
In this section, we present two greedy algorithms to tackle Problem~\ref{prob:problem1} and Problem~\ref{prob:problem2} and prove that the algorithms have constant factor approximation bounds and run in polynomial time.

\subsection{1/2–Approximation Algorithm for Problem~1}
We first study the assignment with sufficient sensing capability (Problem \ref{prob:problem1}) where a single robot is capable of estimating the state of a target. 
% The goal is to assign robot-actions to targets such that the tracking quality $q(\phi_1(j), j)$ is maximized. 
We propose a greedy algorithm in Algorithm~\ref{algorithm:complete_assignment} where $q(\texttt{GREEDY})$ denotes the total tracking quality obtained by the greedy algorithm. In each round, we compute the tracking quality for all pairs of robot-action and target $q(\phi_1(j), j), ~\phi(j)\in \mc{A}, j\in\mc{T}$, and select the pair that has the maximum $q(\phi_1(j), j)$. Then we remove the robot $i(\phi_1(j))$ (that the action $\phi_1(j)$ belongs to) from robot set $\mc{R}$, remove robot $i(\phi_1(j))$'s action set $\mc{A}_{i(\phi_1(j))}$  from the joint action set $\mc{A}$, and remove target $j$ from target set $\mc{T}$. This is due to the constraints that each robot-action can be assigned to at most one target and each robot executes one action per step. The assignment is complete when all targets are tracked, \textit{i.e.,} $\mc{T} = \emptyset$.

% Therefore, the first step of the algorithm is to estimate the tracking quality of all the robots taking all the candidate actions for all the targets. Using the greedy algorithm, we pick the duo of the robot-action pair and target $q(\phi(j), j)$, which gives the best tracking quality metric, as the action selection and target assignment. Then, we remove that robot-action pair $i$ from the robot-action pair $\mc{R}_\mc{U}$, and target $j$ from the target set $\mc{T}$, respectively. Finally, we use the remaining candidate sets to repeat the operation until the target set $\mc{T} = \emptyset$, which means the assignment is completed.
%%%%%%%%%%%%%% Start of Algorithm 1 %%%%%%%%%%%%%%
\begin{algorithm}
\caption{Greedy Robot-action Assignment}
$h\leftarrow 0, ~q(\texttt{GREEDY})\leftarrow 0$\\
\While{true}{
Compute all possible 
$q(\phi_1(j), j)$.\\
Select the action and target pair $(\phi_1(j), j)$ with maximum $q(\phi_1(j), j)$ defined as $q_{\max}$.\\ $q(\texttt{GREEDY})\leftarrow q(\texttt{GREEDY})+q_{\max}$.\\
% $\mathcal{S}\backslash\{s_i,s_j\}$ and $\mathcal{T}\backslash\{t_l\}$ $\leftarrow$ 
Remove $i(\phi_1(j))$ from the robot set $\mc{R}$, remove $\mc{A}_{i(\phi_1(j))}$ from joint action set $\mc{A}$, and remove $j$ from the target set $\mathcal{T}$.\\
$h\leftarrow h + 1$
}   
 \label{algorithm:complete_assignment}
\end{algorithm}
%%%%%%%%%%%%% End of Algorithm 1 %%%%%%%%%%%%%%%%%%

\begin{thm}\label{thm:12thm}
    $q(\texttt{GREEDY}) \geq \frac{1}{2}q(\texttt{OPT})$ where \texttt{OPT} denotes the optimal algorithm for problem \ref{prob:problem1}. The running time for Algorithm \ref{algorithm:complete_assignment} is $O(|\mc{A}|M^2).$ 
\end{thm}
% The proof is included in the full version of this paper \cite{li2023assignment}.
\begin{proof}
\label{proof:12proof}
    % We provide the analysis and proof of the presented algorithm \ref{algorithm:complete_assignment} and \ref{algorithm:limited_assignment}.
    We first prove the $1/2$ approximation bound of Algorithm~\ref{algorithm:complete_assignment}. The proof builds on the results in our previous work~\cite[Sec. IV-B]{zhou2019sensor}. Recall that $q(\texttt{GREEDY})$ and $q(\texttt{OPT})$ denote the sum of tracking quality of pairs consisting of one target and one assigned action. As a shorthand, we use $q^*(j)$ and $q^{g}(j)$ to denote the tracking quality of the pair assigned to target $j$ by \texttt{OPT} and \texttt{GREEDY}, respectively. 
    
    % We denote $q(\texttt{GREEDY})$ and $q(\texttt{OPT})$ as the sum of $q(\cdot)$ metric term of the action assigned to the targets, and \texttt{OPT} is using the optimal algorithm. We also denote $q^*(j)$ and $q^{g}(j)$ as the metrics of the action assigned to $j$ by \texttt{OPT} and \texttt{GREEDY}, respectively.

    We show that there exists a many-to-one mapping $\mc{F}: [1,\cdots, M] \to [1,\cdots, M]$ such that: 
    \begin{enumerate}
        \item $q^{*}(j)\leq q^{g}(\mc{F}(j))$; and 
        \item $|\mc{F}^{-1}(y)|\leq 2$ for all $y\in \mc{Y}$ where $\mc{Y} \subseteq [1,\cdots, M]$ is the range of $\mc{F}$.
    \end{enumerate}
    This mapping conveys that each pair in \texttt{OPT} is mapped to a pair in \texttt{GREEDY} whose tracking quality is at least as high and no pair in \texttt{GREEDY} has more than two items in \texttt{OPT} mapped to it. 
    % Moreover, none of the pairs in \texttt{GREEDY} contain more than two terms that have been mapped from \texttt{OPT}.
    % \footnote{recall that action means robot-action pair as described in Sec. \ref{subsec: framework}}
    We first show that the approximation ratio holds if such a mapping exists. Then, we prove the existence of such a mapping by constructing a specific one. 
    
    We prove $q(\texttt{GREEDY}) \geq \frac{1}{2}q(\texttt{OPT})$ if such a mapping $\mc{F}$ exists as below: 
    \begin{align}
        q(\texttt{OPT}) &= \sum_{j=1}^{M} q^{*}(j) \leq \sum_{j=1}^{M} q^{g}(\mc{F}(j))\nonumber\\
        &= \sum_{y\in \mathcal{Y}} q^{g}(y) |\mc{F}^{-1}(y)| \nonumber\\
        &\leq 2 \sum_{y\in \mathcal{Y}} q^{g}(y) \leq 2 \sum_{m=1}^{M} q^{g}(m)\nonumber\\
        &= 2 q(\texttt{GREEDY}).
    \label{eqn:1_2approximation}
    \end{align}
     The first inequality holds because $q^{*}(j)\leq q^{g}(\mc{F}(j))$. The second equality holds because the mapping $\mc{F}$ maps each item in $[1,\cdots, M]$ to set $\mc{Y}$. The second inequality holds because $|\mc{F}^{-1}(y)|\leq 2$. The third inequality holds because $\mc{Y} \subseteq [1,\cdots, M]$.
    
    Next, we construct such a mapping to show it always exists. We will define $\mc{M}$ in the order in which the pairs are selected by \texttt{GREEDY}. Assume \texttt{GREEDY} selects the pair $(\phi_i, j)$ in a round. Then we have $q^g(j) = q(\phi_i, j)$. We will then map at most two pairs in \texttt{OPT} to the pair $(\phi_i, j)$ in \texttt{GREEDY}, which is described in the following two cases.
    
    % The order of  $\mc{F}$ follows the order of duo selected by \texttt{GREEDY}. When a duo $(\phi_a, j)$ is selected by \texttt{GREEDY} in a particular round, we define $q^g(j) = q(\phi_a, j)$. The mapping assigns a maximum of two duos in \texttt{OPT} to this selected duo in \texttt{GREEDY}. There are two possible cases that we consider in this mapping.

    \begin{figure}[tbh]
    \centering{
    \subfigure[Case 1]{
    \includegraphics[width=0.35\columnwidth]{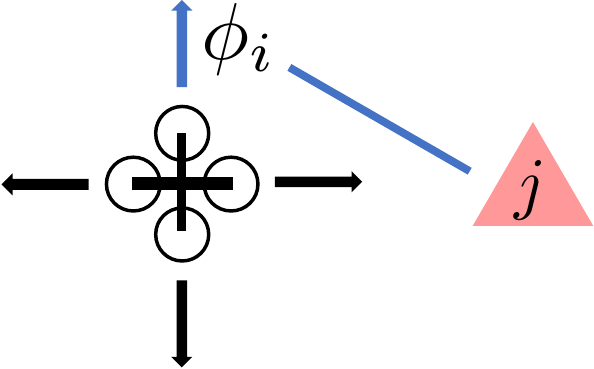}} ~~
    \subfigure[Case 2]{
    \includegraphics[width=0.45\columnwidth]{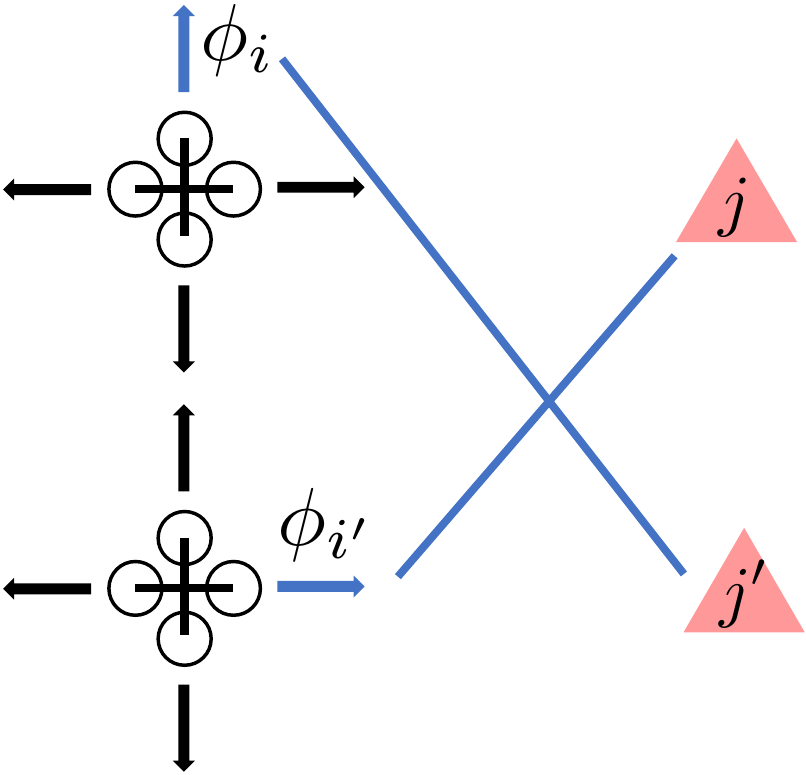}}
    }
    \caption{The optimal solution in two cases. \revo{In all cases,  \texttt{GREEDY} selects pair $(\phi_i,j)$. } Case 1: \texttt{OPT} charges $q(\phi_i,j)$ to the same pair $(\phi_i,j)$ selected by the \texttt{GREEDY}.  Case 2: \texttt{OPT} charges $q(\phi_i,j)$ to at most two pairs --- $(\phi_i,j')$ and  $(\phi_{i'},j)$. \label{fig:opt_choose}}
    \end{figure}
    
    % \begin{figure}
    %     \centering{
    %     \subfigure[Case 1]{
    %     \includesvg[width=0.3\columnwidth]{figs/assignment-1-1-1v2.svg}}
    %     \subfigure[Case 2]{
    %     \includesvg[width=0.3\columnwidth]{figs/assignment-1-1-2v2.svg}}
    %     }
    %     \caption{Here shows the two cases assigned by \texttt{OPT}. }
    %     \label{fig:alg1}
    % \end{figure}
    
    \begin{enumerate}
        \item \texttt{OPT} also selects $(\phi_i, j)$(Fig.~\ref{fig:opt_choose}-(a)). If $\mc{F}(j)$ has not been defined in prior rounds, we define $\mc{F}(j) = j$. Notably, here, $q^g(j)=q^*(j)$ and $|\mc{F}^{-1}(j)| = 1$. That way, the two conditions for a valid mapping are met.
    
        \item One of $(\phi_i, j)$ appears in a pair selected by \texttt{OPT} (Fig.~\ref{fig:opt_choose}-(b)). As an example, \texttt{OPT} selects $(\phi_i, j')$ and $(\phi_{i'}, j)$ where $i\neq i'$ and $j' \neq j$. If $\mc{F}(j')$ has not been defined in prior rounds, we define $\mc{F}(j') = j$. Notably, if $\mc{F}(j')$ was not defined in a previous round, $q^*(j')\leq q^g(j)$. Otherwise, \texttt{GREEDY} would select the pair $(\phi_i, j')$ in this round. Similarly, if $\mc{F}(j)$ has not been defined in prior rounds, we define $\mc{F}(j) = j$. With a similar argument, if $\mc{F}(j)$ was not defined in a previous round, $q^*(j)\leq q^g(j)$. Otherwise, \texttt{GREEDY} would select the pair $(\phi_i', j)$ in this round. Thus, in this case,  $|\mc{F}^{-1}(j)| = 2$. That way, the two conditions for a valid mapping are met.  
        
        % \texttt{OPT} makes selection that only contain one element from the duo $(\phi_a, j)$. An example of such case is the \texttt{OPT} selects $(\phi_a, k)$ and $(\phi_b, j)$ where $k \neq j$. We define $\mc{F}(k)=j$ if we have not previously defined it, and $q^*(k) \leq q^g(j)$. This is because if this in-equation doesn't hold, then \texttt{GREEDY} would pick the same duo $(\phi_a, k)$ in the current round. In the same way, if we have not defined $\mc{F}(j)$ before, we define $\mc{F}(j)=j$ and $q^*(j) \leq q^g(j)$. Also, $|\mc{F}^{-1}(j) \leq 2$ in this scenario. Therefore, both conditions are satisfied, and the existence of mapping $\mc{F}$ is proved.
    \end{enumerate}
    Therefore, with such a mapping $\mc{F}$, it holds that $q(\texttt{GREEDY}) \geq \frac{1}{2}q(\texttt{OPT})$.
    
    Next, we prove the running time for Algorithm \ref{algorithm:complete_assignment}. The ``while" loop takes $M$ rounds since all targets must be tracked in the end. Inside each round of the ``while" loop, all possible pairs are computed and the best one is selected, which takes $O(|\mc{A}|M)$ time. Overall, Algorithm \ref{algorithm:complete_assignment} runs in $O(|\mc{A}|M^2)$ time.
\end{proof}

\subsection{1/3–Approximation Algorithm for Problem~2}
Next, we study the assignment with limited sensing capability (Problem \ref{prob:problem2}) where at least two robots are required to estimate the state of a target. The goal is to assign non-overlapping pairs of robot-action to targets such that the tracking quality $q(\{\phi_1(j), \phi_2(j)\}, j)$ is maximized. We propose a greedy algorithm in Algorithm~\ref{algorithm:limited_assignment}. In each round, we compute the tracking quality for all triples consisting of two actions and one target $q(\{\phi_1(j), \phi_2(j)\}, j), ~\phi_1(j), \phi_2(j), \in \mc{A}, j\in\mc{T}$, and select the triple that has the maximum $q(\{\phi_1(j), \phi_2(j)\}, j)$. Then we remove the robot $i(\phi_1(j))$ (that the action $\phi_1(j)$ belongs to) and $i(\phi_2(j))$ (that the action $\phi_2(j)$ belongs to) from robot set $\mc{R}$, remove robot $i(\phi_1(j))$'s action set $\mc{A}_{i(\phi_1(j))}$ and robot $i(\phi_2(j))$'s action set $\mc{A}_{i(\phi_2(j))}$ from the joint action set $\mc{A}$, and remove target $j$ from target set $\mc{T}$. This is because each robot-action can be assigned to at most one target and each robot executes one action per step. The assignment is complete when all targets are tracked, \textit{i.e.,} $\mc{T} = \emptyset$.

% For this problem, we propose a greedy algorithm for the solution. In each iteration, we estimate the tracking quality of all the triples, $q(\phi_1(j), \phi_2(j), j)$, after the action was taken for the robots. Then, the greedy algorithm will select the triple with maximum tracking quality metric and remove the robot-action ${\phi_1(j), \phi_2(j)}$ from the robot-action set $\mc{R}_\mc{U}$, and target $j$ from the target set $\mc{T}$, respectively.

%%%%%%%%%%%%%% Start of Algorithm 2 %%%%%%%%%%%%%%
\begin{algorithm}
\caption{Greedy Robot-action Pair Assignment}
$h\leftarrow 0, ~q(\texttt{GREEDY})\leftarrow 0$\\
\While{true}{
Compute all possible 
$q(\{\phi_1(j),\phi_2(j)\}, j)$.\\
Select the triple $(\{\phi_1(j),\phi_2(j)\}, j)$ with maximum $q(\{\phi_1(j),\phi_2(j)\}, j)$ defined as $q_{\max}$.\\ $q(\texttt{GREEDY})\leftarrow q(\texttt{GREEDY})+q_{\max}$.\\
Remove $i(\phi_1(j)),i(\phi_2(j))$ from the robot set $\mathcal{R}$, remove $\mc{A}_{i(\phi_1(j))}$, $\mc{A}_{i(\phi_2(j))}$ from joint action set $\mc{A}$, and remove $j$ from the target set $\mathcal{T}$.\\
$h\leftarrow h + 1$
}   
 
 \label{algorithm:limited_assignment}
\end{algorithm}
%%%%%%%%%%%%% End of Algorithm 1 %%%%%%%%%%%%%%%%%%

\begin{thm}\label{thm:13thm}
    $q(\texttt{GREEDY} \geq \frac{1}{3}q(\texttt{OPT})$ where \texttt{OPT} denotes the optimal algorithm for Problem \ref{prob:problem2}. The running time for Algorithm \ref{algorithm:limited_assignment} is $O(|\mc{A}|^2M^2)$. 
\end{thm}
\begin{proof}
    \label{proof:13proof}
    We first prove $1/3$ approximation bound of Algorithm~\ref{algorithm:limited_assignment}. Since the proof is similar to that of Theorem~\ref{thm:12thm}, we omit the detailed version but provide a sketch below.

    First, we list three cases (similar to the two cases in the proof of Theorem~\ref{thm:12thm}) and show there exists a many-to-one mapping $\mathcal{F}: [1,\cdots, M] \to [1,\cdots, M]$ such that: 
\begin{enumerate}
        \item $q^{*}(j)\leq q^{g}(\mc{F}(j))$; and 
        \item $|\mc{F}^{-1}(y)|\leq 3$ for all $y\in \mc{Y}$ where $\mc{Y} \subseteq [1,\cdots, M]$ is the range of $\mc{F}$.
\end{enumerate}
Here $q^{*}(j)$ and $q^{g}(\mc{F}(j))$ denote the tracking quality of the triple assigned to target $j$  by \texttt{OPT} and to target $\mc{F}(j)$ by \texttt{GREEDY}, respectively. Second, following the steps in Equation~\ref{eqn:1_2approximation} and replacing the scalar $2$ by $3$, we reach the $3$ approximation. 

    Then, we prove the running time for Algorithm \ref{algorithm:limited_assignment}. Similarly, the ``while" loop takes $M$ rounds since all targets must be tracked eventually. Inside each round of the ``while" loop, all possible triples are computed and the best one is selected, which takes $O(|\mc{A}|^2M)$ time. Overall, Algorithm \ref{algorithm:limited_assignment} runs in $O(|\mc{A}|^2M^2)$ time. 
    % to Algorithm \ref{algorithm:complete_assignment}, there are $M$ rounds in the while loop as all the targets need to be tracked. In each round, the computation for all the combinations of robot-action pair and target is $K^2N^2M$ times, considering every robot has the same action set. Therefore, the total running time for Algorithm \ref{algorithm:limited_assignment} is $O(K^2N^2M^2)$.
\end{proof}
% The proof is included in the full version of this paper \cite{li2023assignment}.

\begin{rem}\label{rem:rem1}
% The results in Theorem~\ref{thm:12thm} and \ref{thm:13thm} can be generalized to the case where exactly $n$ robot-actions are assigned to a target with $n\geq 2$. We can derive a generalized bound, $\omega(\emph{\texttt{GREEDY}}) \geq \frac{1}{n+1}\omega(\emph{\texttt{OPT}})$ by using a proof sketch similar to that of Theorem~\ref{thm:13thm}. Please refer to the full version of this paper \cite{li2023assignment} for details.
\end{rem}

\begin{proof}
    First, we list $n+1$ cases (similar to the two cases in the proof of Theorem~\ref{thm:12thm}) and show there exists a many-to-one mapping $\mathcal{F}_n: [1,\cdots, M] \to [1,\cdots, M]$ such that: 
\begin{enumerate}
        \item $q^{*}(j)\leq q^{g}(\mc{F}_n(j))$; and 
        \item $|\mc{F}_n^{-1}(y)|\leq 3$ for all $y\in \mc{Y}$ where $\mc{Y} \subseteq [1,\cdots, M]$ is the range of $\mc{F}_n$.
\end{enumerate}
Here $q^{*}(j)$ and $q^{g}(\mc{F}_n(j))$ denote the tracking quality of the $n+1$-tuple assigned to target $j$  by \texttt{OPT} and to target $\mc{F}_n(j)$ by \texttt{GREEDY}, respectively. Second, following the steps in Equation~\ref{eqn:1_2approximation} and replacing the scalar $2$ by $n+1$, we reach the $n+1$ approximation for this general case.  
\end{proof}

\begin{rem}\label{rem:rem2}
The bounds (i.e., $1/2$ in Theorem~\ref{thm:12thm}, $1/3$ in Theorem~\ref{thm:13thm}, and $1/(n+1)$ in Remark~\ref{rem:rem1}) obtained by the greedy algorithms hold for any arbitrary metric of tracking quality such as trace, log determinant, or maximum eigen-value of the covariance matrix and do not require additional assumptions on the properties of the metrics such as monotonicity or submodularity~\cite{jawaid2015submodularity}. 
\end{rem}

\section{Simulation}
In this section, we evaluate the proposed greedy algorithms (Algorithm~\ref{algorithm:complete_assignment} and Algorithm~\ref{algorithm:limited_assignment}) with extensive simulations. We first qualitatively illustrate the effectiveness of the two algorithms for multi-target tracking in a ROS-Gazebo environment. Then we quantitatively demonstrate the near-optimal performance of the algorithms through comparisons. The evaluations are performed on a ThinkPad with Intel Core i7 and 32 GB Memory. The code of the evaluations is available online\footnote{\url{https://github.com/Zhourobotics/assignment-target-tracking.git}}. Before that, we introduce the target tracking framework (Section~\ref{subsec: framework}) in more detail. Note that the algorithms are agnostic to the motion models. The motion models introduced here are for simulation purposes only.

% First, we introduce the simulation setup. Then, we show the effectiveness of target tracking in the subsection of the qualitative result. Finally, in the subsection of the quantitative result, we demonstrate the near-optimal performance of our algorithm. 

% \subsection{Simulation Setup}
% The simulation setup follows the same structure as our previous work\cite{zhou2023robust}. 

% We describe more about the robot motion model, target motion model, sensor model, and target tracking objective as follows.
\paragraph{Robot motion model}
% Robots in our simulation follow the framework described in Sec. \ref{subsec: framework}. More specifically, e
Each robot $i \in \mc{R}$ follow the unicycle motion model:
\begin{align*} %\label{eq:robot_motion_model}
\begin{split}
{\begin{pmatrix}
x_{i,t+1}^1 \\ x_{i,t+1}^2 \\ \theta_{i,t+1}
\end{pmatrix} = 
\begin{pmatrix}
x_{i,t}^1  \\ x_{i,t}^2 \\ \theta_{i,t}
\end{pmatrix} + 
\begin{pmatrix}
v_i \Delta T \cos(\theta_{i,t})\\
v_i \Delta T \sin(\theta_{i,t})\\
\Delta T \omega_i
\end{pmatrix},}
\end{split}
\end{align*}
where the state of the robot is represented as $\mb{x}_i = [x_{i,t}^1, x_{i,t}^2, \theta_{i,t}]^\top$ with $ [x_{i,t}^1, x_{i,t}^2]^\top$ as the position in the 2D plane and $\theta_{i,t}$ as the orientation of the robot relative to the world frame. $\mb{a}_i = [v_i, \omega_i]^\top$ represents the robot's action (or control input) where $v_i$ and $\omega_i$ are the linear and angular velocities, respectively. $\Delta T$ denotes the time interval between two consecutive time steps. 

\paragraph{Target motion model}
Each target $j \in \mc{T}$ follows a circular motion model with the added white Gaussian noise:
\begin{align*} %\label{eq:target_motion_model}
\begin{split}
\begin{pmatrix}
y_{j,t+1}^1 \\ y_{j,t+1}^2
\end{pmatrix} = 
\begin{pmatrix}
y_{j,t}^1  \\ y_{j,t}^2
\end{pmatrix} + 
\begin{pmatrix}
v_j \cos(\Delta T \omega_j)\\
v_j \sin(\Delta T \omega_j)
\end{pmatrix} + \mb{w}_{j,t}, 
\end{split}
\end{align*}
where the target position in 2D space is denoted as $\mb{y}_{j, t} = [y_{j,t}^1, y_{j,t}^1]^\top$. $[v_j, \omega_j]$ denotes the linear and angular velocity of the target. The white Gaussian noise is $\mb{w}_{j,t} \sim \mathcal{N}(0, \mb{Q})$ with 
$$
\mb{Q} = \begin{bmatrix}
\sigma_{j}^2  & 0  \\ 0 & \sigma_{j}^2
\end{bmatrix}.
$$
\paragraph{Sensor model}
There are two kinds of sensors---range sensor and bearing sensor. The range sensor has the observation model:
\begin{equation*}
    r(\mb{x}_{i,t}, \mb{y}_{j,t}) = \sqrt{(y_{j,t}^2- x_{i,t}^2)^2 + (y_{j,t}^1 - x_{i,t}^1)^2}
\end{equation*}
with a zero mean white Gaussian measurement noise $\sigma_{r}^2(\mb{x}_{i,t}, \mb{y}_{j,t})$. The bearing sensor has the observation model:
\begin{equation*}
    \gamma(\mb{x}_{i,t}, \mb{y}_{j,t}) = \text{atan2}(y_{j,t}^2- x_{i,t}^2, y_{j,t}^1 - x_{i,t}^1) - \theta_{i,t}
\end{equation*}
with a zero mean white Gaussian measurement noise $\sigma_{b}^2(\mb{x}_{i,t}, \mb{y}_{j,t})$. For both measurement models, the noise increases linearly as the distance between robot $i$ and target $j$ increases. 

In the sufficient sensing scenario (Problem \ref{prob:problem1}), each robot has a range-and-bearing sensor. In this case, we combine the range and bearing observations to estimate the target position. While in the limited sensing scenario, each robot can only collect range \textit{or} bearing measurement. In this case, we stack two range observations or two bearing observations 
 from two robots for the estimation.
% \begin{align*}\label{eq:measure_model_range_bearing}
% %  \mb{z}_{i,t}^j =  h_i^j(\mb{x}_{i,t}, \mb{y}_{j,t}) + \mb{v}_{i,t}^j(\mb{x}_{i,t}, \mb{y}_{j,t}). \\
% & z_i^j(\mb{x}_{i,t}, \mb{y}_{j,t})  = \begin{bmatrix}r(\mb{x}_{i,t}, \mb{y}_{j,t})\\\gamma(\mb{x}_{i,t}, \mb{y}_{j,t})\end{bmatrix} \\
% &\triangleq \begin{bmatrix} \sqrt{(y_{j,t}^2- x_{i,t}^2)^2 + (y_{j,t}^1 - x_{i,t}^1)^2} \\ \text{atan2}(y_{j,t}^2- x_{i,t}^2, y_{j,t}^1 - x_{i,t}^1) - \theta_{i,t}
% \end{bmatrix} 
% \end{align*}
% The measurement noise $\mb{v}_{i,t}^j(\mb{x}_{i,t}, \mb{y}_{j,t}) \sim \mc{N}(0, \mb{R}(\mb{x}_{i,t}, \mb{y}_{j,t}))$ with 
% $$
% \mb{R}(\mb{x}_{i,t}, \mb{y}_{j,t})) = \begin{bmatrix}
% \sigma_{r}^2(\mb{x}_{i,t}, \mb{y}_{j,t})  & 0  \\ 0 & \sigma_{b}^2(\mb{x}_{i,t}, \mb{y}_{j,t})
% \end{bmatrix}, 
% $$
% where we denote the noise covariance of range and bearing sensors as $\sigma_{r}^2(\mb{x}_{i,t}, \mb{y}_{j,t})$ and $\sigma_{b}^2(\mb{x}_{i,t}, \mb{y}_{j,t})$, respectively. 

% Since we use EKF for state estimation, the linearized nonlinear sensor model around the predicted target position is:
% \begin{align*}
% & \nabla_{\mb{y}_j} h_i^j(\mb{x}_{i}, \mb{y}_{j}) = \\& {\frac{1}{r(\mb{x}_{i,t}, \mb{y}_{j,t})} 
% \begin{bmatrix} (y_{j}^1 - x_{i}^1) & (y_{j}^2 - x_{i}^2) & 0_{1\times2} \\
% -\sin( \theta_i + \gamma(\mb{x}_{i},\mb{y}_{j})) & \cos (\theta_i + \gamma(\mb{x}_{i},\mb{y}_{j})) & 0_{1\times2} \end{bmatrix}.} 
% \end{align*}

\paragraph{Objective function}
We evaluate the tracking quality by the reduction in the trace of the covariance matrix~\cite{zhou2023robust}:
\begin{equation*}
    q =\texttt{Tr}(\mb{\Sigma}_{\mc{T}, t-1|t}) - \texttt{Tr}(\mb{\Sigma}_{\mc{T}, t|t}),
\end{equation*}
where $\mb{\Sigma}_{\mc{T}, t-1|t}$ is the a prior covariance matrix and $\mb{\Sigma}_{\mc{T}, t|t}$ is the a posteriori covariance matrix from EKF. Our framework and proposed algorithms are agnostic to the metrics of tracking quality, as stated in Remark \ref{rem:rem2}.

\subsection{Qualitative Result}
\label{subsec:quali}

We evaluate the performance of the greedy algorithms (Algorithm~\ref{algorithm:complete_assignment}
and Algorthm~\ref{algorithm:limited_assignment}) using ROS-Gazebo simulator on Ubuntu 20.04. We use the Hummingbird drones as robots to track the Scarab cars (targets) in a $20\times20 ~m^2$ environment. 
The robots have the same candidate action set with $\mb{a}_i=\{0, \pm1.5\}~m/s \times\{0, \pm0.7\}~rad/s$ for each robot $i$. Each target $j$ follows a circular motion with a linear velocity of $\mb{u}_j=1.2~m/s$ and a angular velocity chosen from $\{0.15, 0.2, 0.3, 0.6\}~rad/s$. 
To avoid the issue of collision avoidance, we let the robots (drones) fly at different altitudes. Top views of the target tracking in action are shown in Figure~\ref{fig:alg1_sim} and Figure~\ref{fig:alg2_sim}. A video of the Gazebo simulations is available online\footnote{\url{https://youtu.be/5w5XJefu7sE}}.

% \LZ{A video of the Gazebo simulations is available online/attached~\footnote{\url{youtube}}.} 

\begin{figure*}[t]
\centering{
\subfigure[$t = 0$]{\includegraphics[width=0.483\columnwidth]{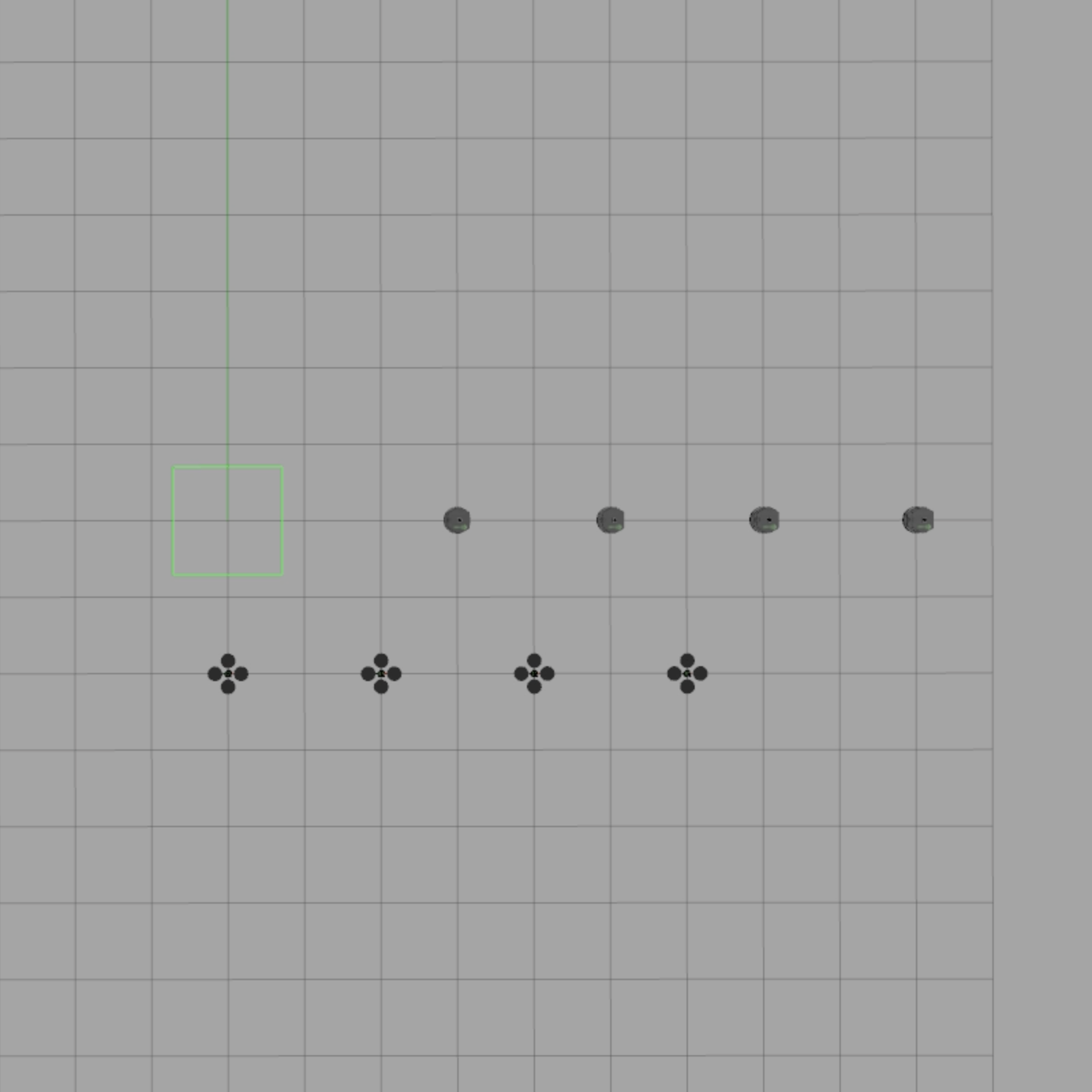}}~~~~~
\subfigure[$t = 30$]{\includegraphics[width=0.483\columnwidth]{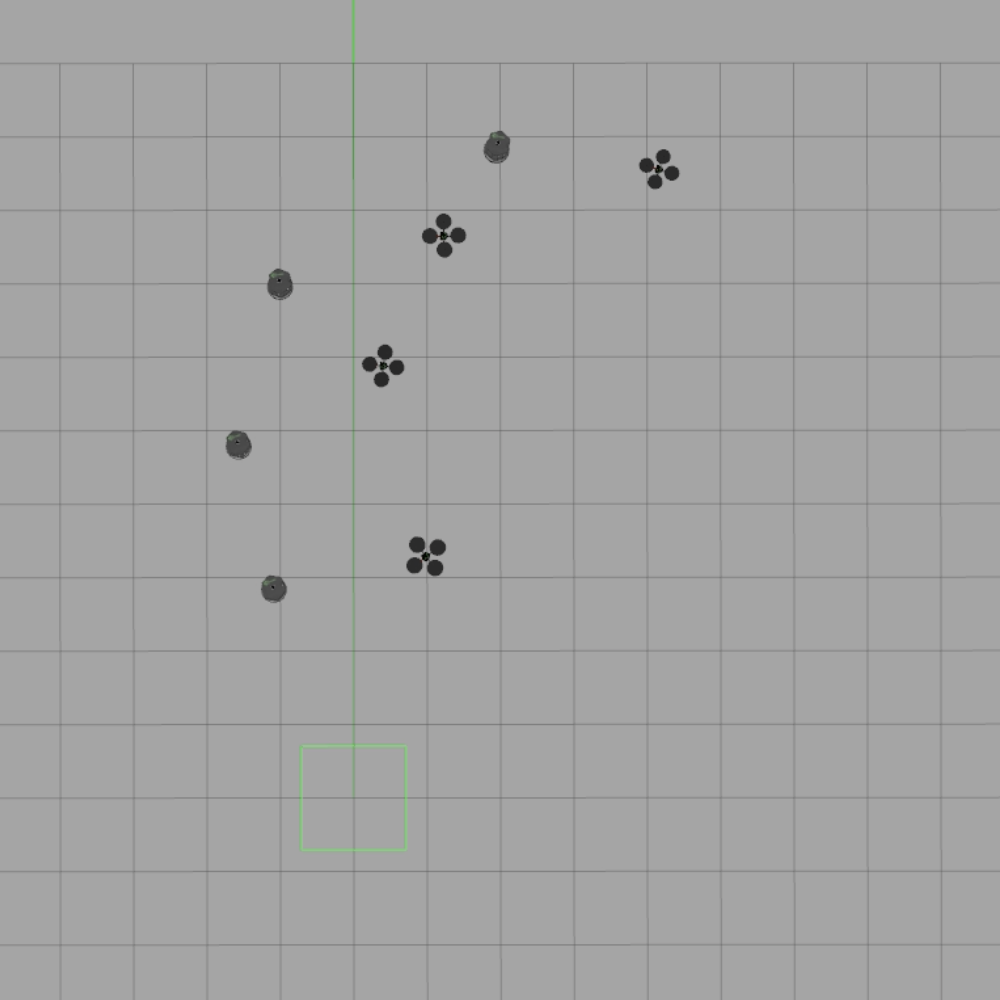}}~~~~~
\subfigure[$t = 60$]{\includegraphics[width=0.483\columnwidth]{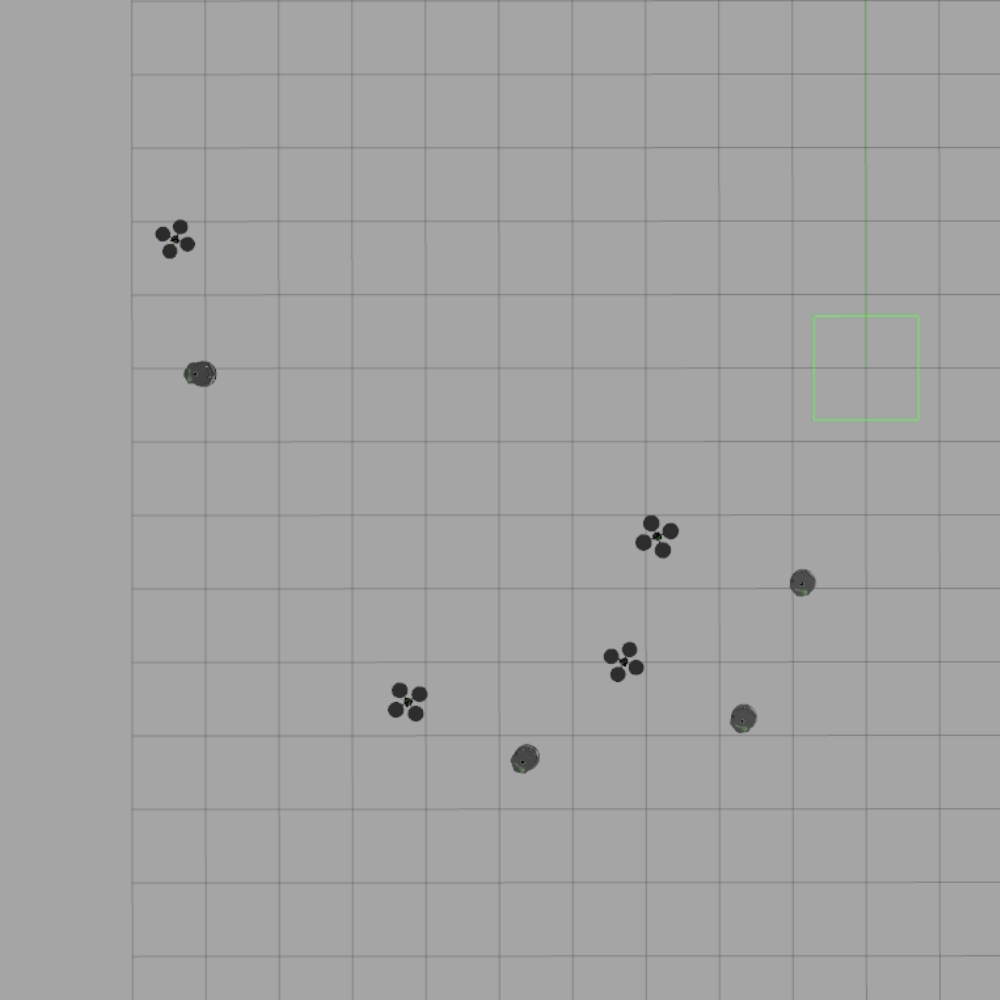}}
}
\caption{A top view of target tracking in action using Algorithm \ref{algorithm:complete_assignment} for the assignment with sufficient sensing in Gazebo environment. The Hummingbird drones are the robots and the Scarab cars (solid dots) are the targets. A robot can sufficiently track a target.}
\label{fig:alg1_sim}
\end{figure*}

\begin{figure}[t]
    \centering
    \includegraphics[width=0.79\columnwidth]{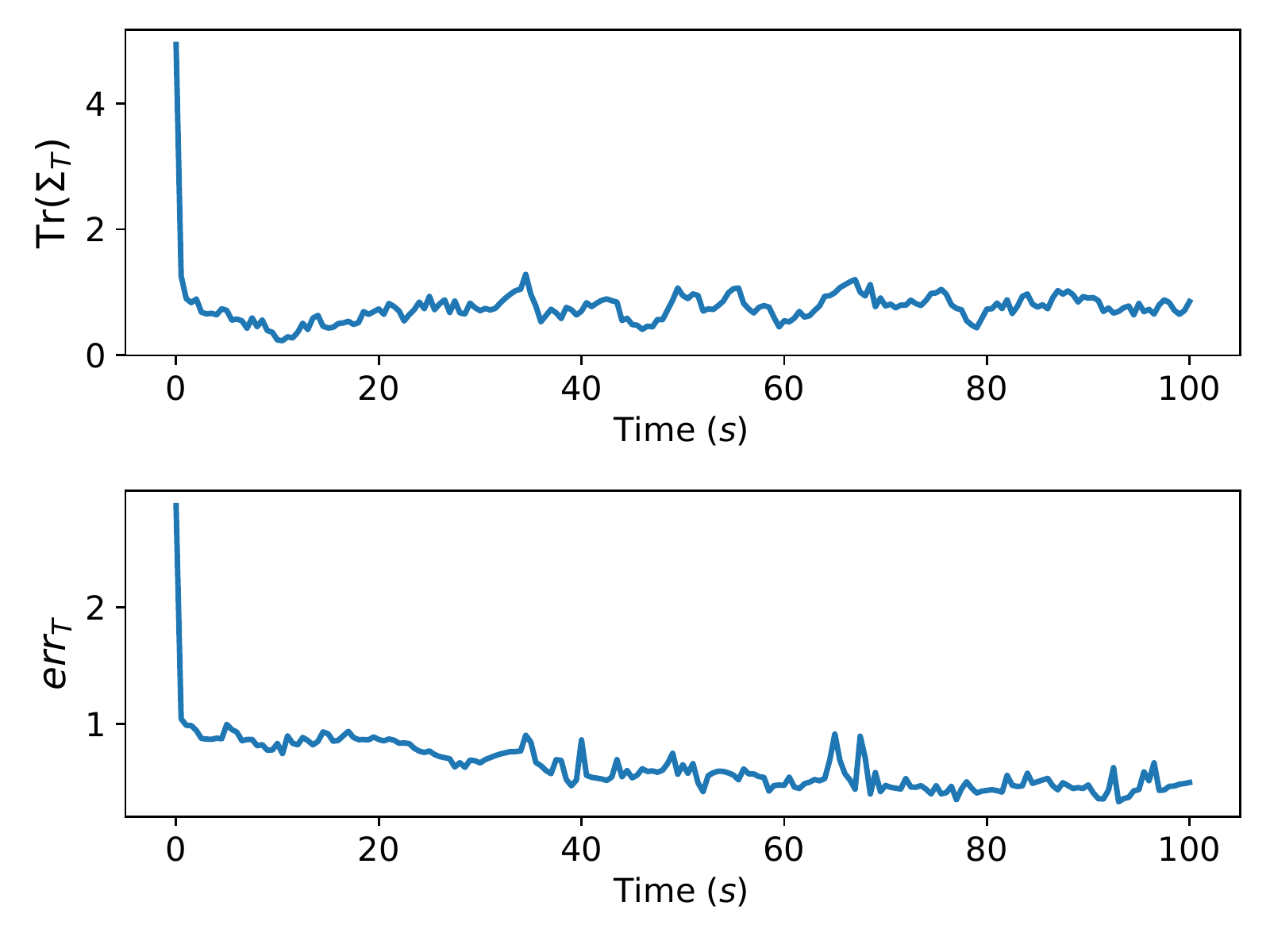}
    \caption{The trace of the covariance matrix (top) and the mean squared estimation error (bottom) of three targets during 100 time steps using Algorithm \ref{algorithm:complete_assignment}. Each robot is equipped with a range-and-bearing sensor.}
    \label{fig:alg1_err}
\end{figure}

\paragraph{Sufficient Sensing}
We first illustrate the performance of Algorithm~\ref{algorithm:complete_assignment} for solving the assignment problem with sufficient sensing (Problem~\ref{prob:problem1}). We use $N=4$ robots, each with a range-and-bearing sensor, to track $M=4$ targets since a single robot is sufficient to estimate the target state in this case. Snapshots of the target tracking at time steps $t=0, 30, 60$ are shown in Figure~\ref{fig:alg1_sim}. 
% We simulate with $M = 4$ and $N = 4$ with the proposed greedy algorithm (Algorithm \ref{algorithm:complete_assignment}). The simulation lasts 100$s$, and each time step is 0.5$s$. Figure \ref{fig:alg1_sim} shows the simulation from Gazebo. 
Here, each target is moving on a circle, centered at the centre of the Gazebo world (represented as the green box in the figures). Notably, Algorithm~\ref{algorithm:complete_assignment} assigns robot-actions to targets to maximize the overall tracking quality, which is defined by the reduction in the trace of the covariance matrix. With  Algorithm~\ref{algorithm:complete_assignment}, the robots are able to keep following and tracking the targets. The pattern of one-to-one assignment can be clearly observed from Figure~\ref{fig:alg1_sim}. Moreover, we plot the trace of the covariance matrix and mean squared estimation errors of the targets in Figure~\ref{fig:alg1_err}. The mean squared estimation error for the targets is defined by $\mathrm{err}_{T} = \frac{1}{M}\sum_{j=1}^{M}\|\hat{y}_{j}-y_{j}\|_2$ with $\hat{y}_{j}$ and $y_{j}$ as the estimated position and true position of each target $j$, respectively. 
% for the targets and the mean error of the targets' position from the estimation of EKF
 % Figure \ref{fig:alg1_err} shows the trace of the covariance matrix for the targets and the mean error of the targets' position from the estimation of EKF. From the EKF, we get the estimated position of each target as $\hat{y}_{j}$ and its covariance matrix $\Sigma_{j}$. Therefore, the error here is defined as $\mathrm{err}_{T} = \frac{1}{M}\sum_{j=1}^{M}\|\hat{y}_{j}-y_{j}\|_2$. 
Figure~\ref{fig:alg1_err} shows that Algorithm~\ref{algorithm:complete_assignment} quickly reduces the trace of the covariance matrix and the estimation error, thus achieving and maintaining a high tracking quality.

\paragraph{Limited Sensing}

We then show the performance of Algorithm~\ref{algorithm:limited_assignment} for solving the assignment problem with limited sensing (Problem~\ref{prob:problem2}). In this case, we use $N=6$ robots, each with a range sensor, to track $M=3$ targets since two robots are necessary to track one target.  Snapshots of the target tracking at three different time steps ($t=0, 30, 60$) are shown in Figure~\ref{fig:alg1_sim}. Note that Algorithm \ref{algorithm:limited_assignment} assigns pairs of robot-actions to targets to maximize the overall tracking quality. The assignment pattern of two robots to one target can be clearly observed in Figure \ref{fig:alg2_sim}. In addition, Figure \ref{fig:alg2_err} shows that Algorithm \ref{algorithm:limited_assignment} reduces the trace of the covariance matrix and the estimation error at early time steps and keeps them low as the tracking evolves. These results demonstrate the effectiveness of Algorithm \ref{algorithm:limited_assignment} for assigning pairs of robot-actions to track targets. 

% With a similar setting as the sufficient sensing simulation, we set $M=3$ and $N=6$ since two robots are required to track one target. Each robot uses the range sensor only. Figure \ref{fig:alg2_sim} shows the Gazebo simulation at three different time steps. The greedy algorithm (Algorithm \ref{algorithm:limited_assignment}) assigns robot-actions to targets that maximize the tracking quality. The two-and-one assignment between the robot-actions and the targets can be easily observed from the figure. 
% Figure \ref{fig:alg2_err} shows the trace of the covariance matrix and mean error with the same definition as the previous. The algorithm reduces the trace of the covariance and the estimation error early in the simulation and keeps them low as the simulation continues.

\begin{figure*}[tb]
\centering{
\subfigure[$t = 0$]{\includegraphics[width=0.483\columnwidth]{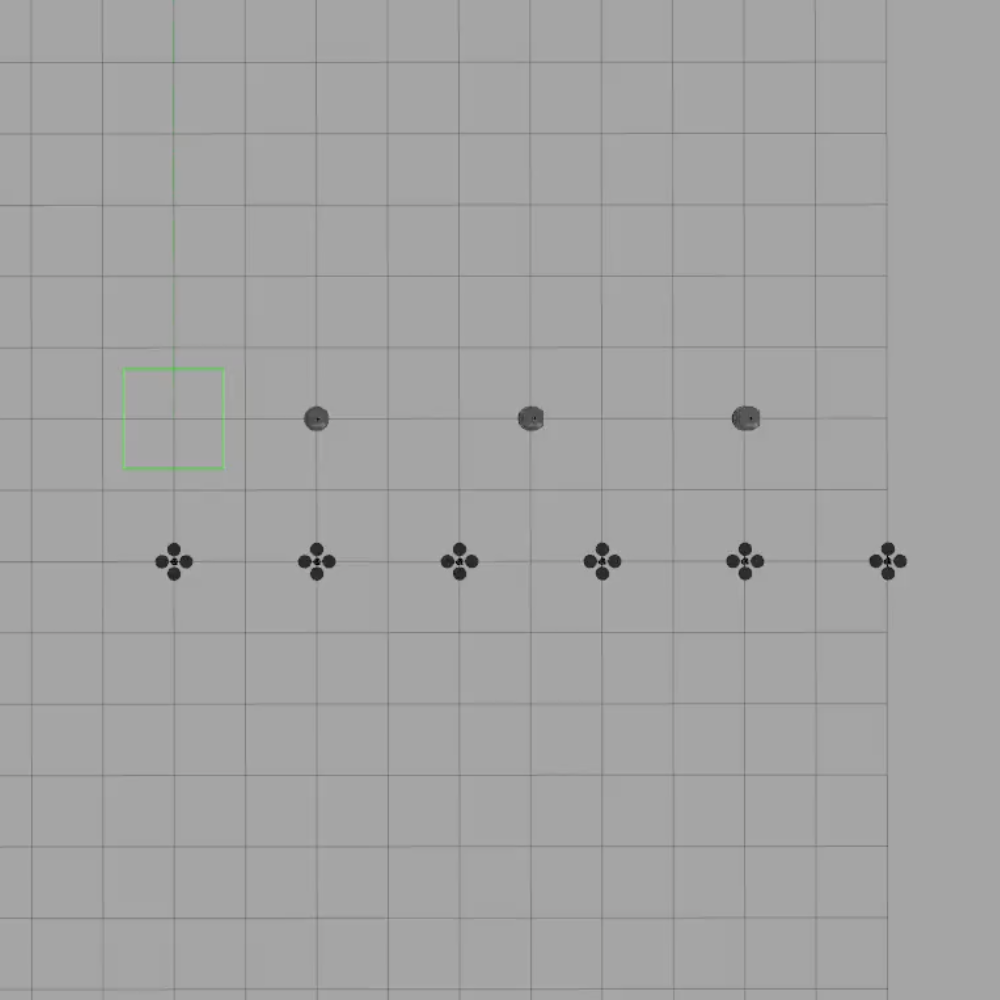}}~~~~~
\subfigure[$t = 30$]{\includegraphics[width=0.483\columnwidth]{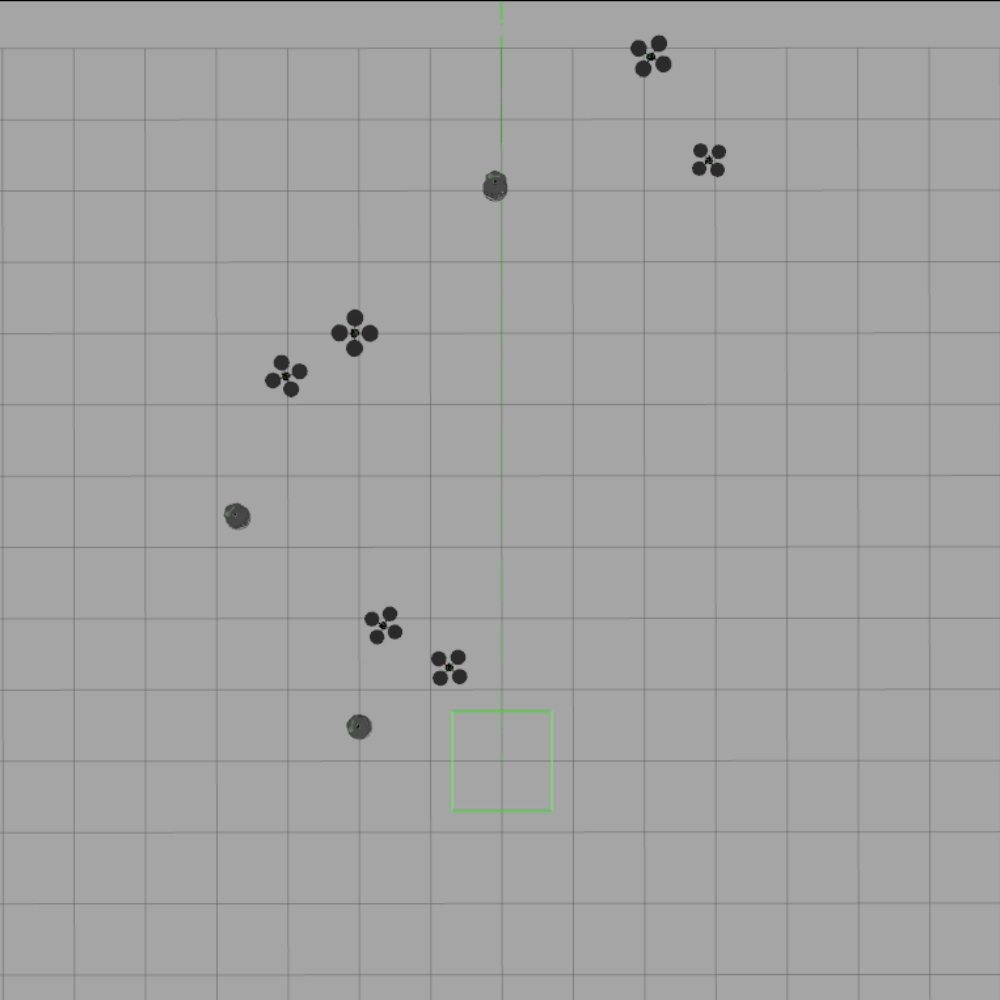}}~~~~~
\subfigure[$t = 60$]{\includegraphics[width=0.483\columnwidth]{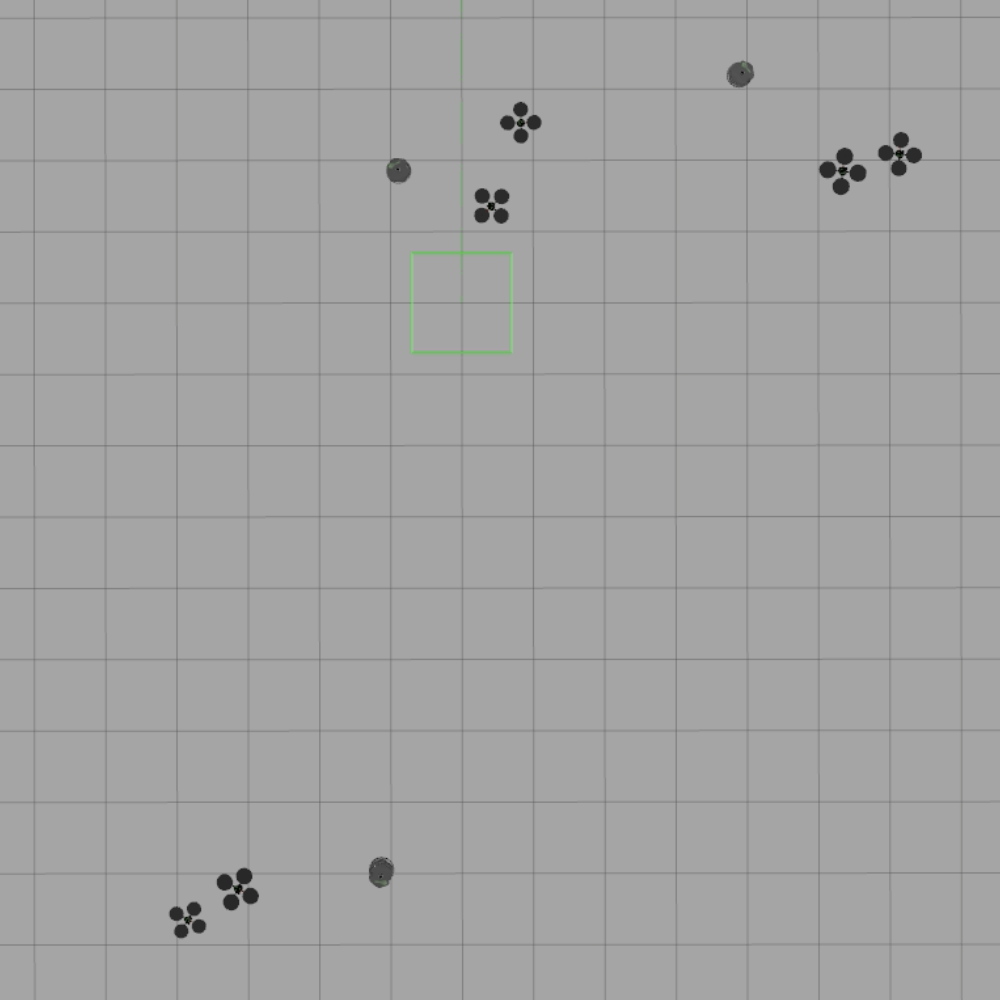}}~~~~~
}
\caption{A top view of target tracking in action using Algorithm \ref{algorithm:limited_assignment} for the assignment with limited sensing in Gazebo environment. The Hummingbird drones are the robots, and the Scarab cars (solid dots) are the targets. Two robots are necessary to track a target.} 
\label{fig:alg2_sim}
\end{figure*}

\begin{figure}[tb]
    \centering
    \includegraphics[width=0.8\columnwidth]{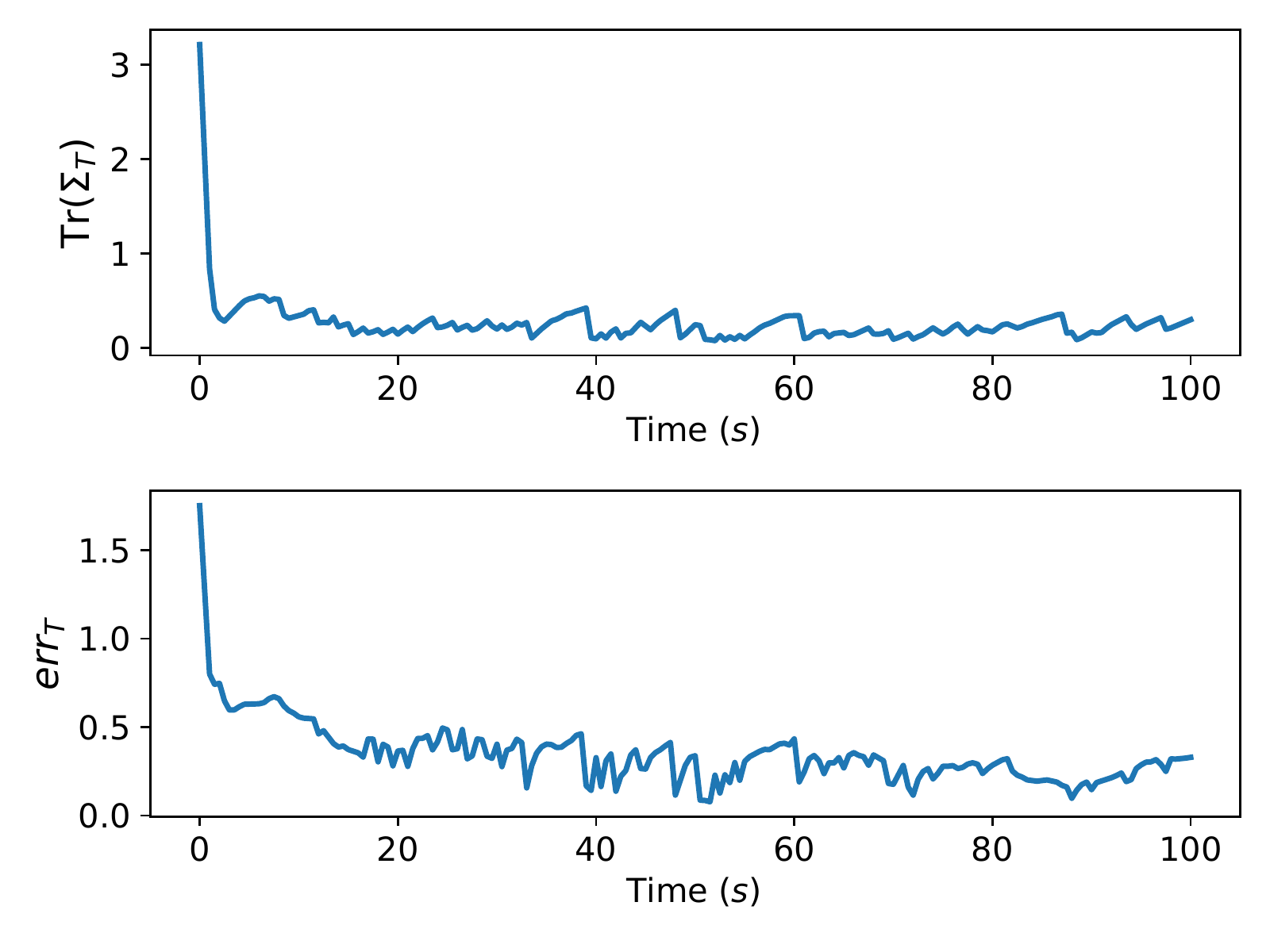}
    \caption{
    The trace of the covariance matrix (top) and the mean squared estimation error (bottom) of three targets during 100 time steps using Algorithm \ref{algorithm:limited_assignment}. Each robot is equipped with a range sensor only.}
    \label{fig:alg2_err}
\end{figure}

\subsection{Quantitative Result}
\label{subsec:quant}

The assignment problems (Problem~\ref{prob:problem1} and Problem~\ref{prob:problem2}) are NP-complete. Thus, it is infeasible to find the optimal solution in polynomial time. To evaluate the proposed greedy algorithms (Algorithm~\ref{algorithm:complete_assignment} and Algorithm~\ref{algorithm:limited_assignment}), we utilize two baseline algorithms. In small-scale cases where the number of robots and targets is small, we use an exhaustive search algorithm to compute the optimal solution. The exhaustive search algorithm enumerates all possible cases to find the best assignment. There are $\prod_{m=0}^{M-1} (N-m)A$ and $\prod_{m=0}^{M-1} \binom{N-2m}{2}A^2$ possibilities in Problem~\ref{prob:problem1} and Problem~\ref{prob:problem2}, respectively, with $A$ denoting the number of actions for each robot. Thus, the exhaustive search algorithm runs in exponential time and is infeasible when the number of robots and targets is large. Therefore, for large-scale cases, we compute an upper bound value of the optimal solution by solving the relaxed versions of Problem~\ref{prob:problem1} and Problem~\ref{prob:problem2}, respectively, which are formally defined in Problem~\ref{prob:problem12} and Problem~\ref{prob:problem22} respectively.   

\begin{problem}
    [Relaxed Assignment with Sufficient Sensing]
    Given a set of robots $\mathcal{R}=\{1, \cdots, N\}$, each with sufficient sensing capability, a set of candidate actions for all robots $\mc{A}$, and a set of targets $\mc{T}=\{1, \cdots, M\}$, find an assignment of robot-actions to targets to
    \begin{align*}
    \begin{split}
        \max \sum_{j=1}^M q(\phi_1(j), j)
        % , i \in \mc{R}, j \in \mc{T};
    \end{split}
    \end{align*}
    with added constraints that each robot-action is assigned to at most one target, i.e., $|\phi^{-1}(i^k)| \leq 1$, assuming $N\geq M$. 
    \label{prob:problem12}
\end{problem}

\begin{problem}
    [Relaxed Assignment with Limited Sensing]
    Given a set of robots $\mathcal{R}=\{1, \cdots, N\}$, each with limited sensing capability, a set of candidate actions for all robots $\mc{A}$, and a set of targets $\mc{T}=\{1, \cdots, M\}$, find an assignment of pairs of robot-actions to targets to
    \begin{align*}
    \begin{split}
        \max \sum_{j=1}^M q(\{\phi_1(j), \phi_2(j)\}, j)
    \end{split}
    \end{align*}
    with added constraints that each robot-action is assigned to at most one target, i.e., $|\phi^{-1}(i^k)| \leq 1$, assuming $N\geq 2M$. 
    \label{prob:problem22}
\end{problem}

Notably, in the relaxed versions, a robot is allowed to select multiple actions per step, which violates the natural constraint and is only for the purpose of comparisons. Clearly, solving the relaxed versions provides us an upper bound of the optimality for Problem~\ref{prob:problem1} and Problem~\ref{prob:problem2}. The upper bound can be used for the comparisons of the greedy algorithms. The relaxed versions can be formulated as the maximum matching problems~\cite{cormen2009introduction} by making a suitable number of copies of the targets and can be solved using the Hungarian algorithm~\cite{kuhn1955hungarian} in polynomial time. We denote the tracking quality obtained by the Hungarian algorithm for the relaxed versions of the problems as $q(\texttt{HUNGARIAN})$.

For comparisons, we randomly generate the positions of the robots and targets for 10 trials with each number of robots and targets. We set $N=L$ and $N=2L$ for sufficient and limited sensing scenarios, respectively. The robots have the same action sets as described in Sec.\ref{subsec:quali}.

% We vary the number of targets and run one step assignment for each target set ten times to get the performance statistics. 

\paragraph{Sufficient Sensing}
% With sufficient sensing capability, one robot is sufficient to track the target. In this setting, each robot equips a range-and-bearing sensor onboard. 
We run comparisons up to $M = N = 8$ since the exhaustive search algorithm runs out of memory with a higher number of robots (or targets) (there are 1,735,643,790,720 
% \LZ{double confirm this with the previous equation} \PL{checked}
possible combinations for each trial). From Fig.~\ref{fig:alg1_baseline}-(a), we observe that $q(\texttt{HUNGARIAN})$ has the best performance since it gives an upper bound of $q(\texttt{OPT})$. The average of $q(\texttt{GREEDY})/q(\texttt{HUNGARIAN})\approx0.92$ and of $q(\texttt{GREEDY})/q(\texttt{OPT})\approx0.98$ for $M = \{1,\cdots,8\}$. The result shows that $q(\texttt{GREEDY})$ stays close to $q(\texttt{OPT})$ and the upper bound $q(\texttt{HUNGARIAN})$, and much higher than the theoretical bound of $\frac{1}{2}q(\texttt{OPT})$, which implies that the Algorithm~\ref{algorithm:complete_assignment} performs even better than the lower bound in practice.

With larger values of  $M$ (or $N$), we compare  $q(\texttt{GREEDY})$ to $q(\texttt{HUNGARIAN})$ without the exhaustive search algorithm. As shown in Fig.~\ref{fig:alg1_baseline}-(b), $q(\texttt{GREEDY})$ is close to $q(\texttt{HUNGARIAN})$ with $q(\texttt{GREEDY})/q(\texttt{HUNGARIAN})\approx0.93$ for $M = \{1,\cdots,50\}$ and much higher than $\frac{1}{2}q(\texttt{HUNGARIAN})$. Therefore, even though Algorithm~\ref{algorithm:complete_assignment} has a theoretical $1/2$--approximation ratio, it works much better in practice. This is because the approximation ratio is derived by considering the worst-case performance of Algorithm~\ref{algorithm:complete_assignment}. 

\paragraph{Limited Sensing}
% With limited sensing, at least two robots are required to track one target. Therefore, we set $N = 2M$. 
For the small-scale comparison, we run algorithms up to $M = 4$ and $N = 8$ because of the exponential running time of the exhaustive search algorithm (there are 108,477,736,920 
% \LZ{check this with the equation} 
possible combinations for each trial). Here, each robot has a range sensor only, which means the robot can take range measurements only. As shown in Fig.~\ref{fig:alg2_baseline}-(a), the average of $q(\texttt{GREEDY})/q(\texttt{HUNGARIAN})\approx0.94$ and $q(\texttt{GREEDY})/q(\texttt{OPT})\approx0.97$ for $M=\{1,\cdots, 4\}$. Moreover, $q(\texttt{GREEDY})$ is close to $q(\texttt{HUNGARIAN})$ and $q(\texttt{OPT})$ and much higher than the theoretical bound of $\frac{1}{3}q(\texttt{OPT})$. This suggests that the Algorithm~\ref{algorithm:limited_assignment}, in practice, performs even better than the lower bound.

With larger values of  $M$ (or $N$), we compare  $q(\texttt{GREEDY})$ to $q(\texttt{HUNGARIAN})$ without the exhaustive search algorithm. As shown in Fig.~\ref{fig:alg2_baseline}-(b), $q(\texttt{GREEDY})$ is close to $q(\texttt{HUNGARIAN})$ with $q(\texttt{GREEDY})/q(\texttt{HUNGARIAN})\approx0.93$ for $M = \{1,\cdots,50\}$ and much higher than $\frac{1}{3}q(\texttt{HUNGARIAN})$. Therefore, even though Algorithm~\ref{algorithm:limited_assignment} has a theoretical $1/3$--approximation ratio, it works much better in practice. This is because the approximation ratio is derived by considering the worst-case performance of Algorithm~\ref{algorithm:limited_assignment}. 

For large-scale cases, we compare  $q(\texttt{GREEDY})$ to $q(\texttt{HUNGARIAN})$ up to $M=25$ and $N=50$. 
% For the large target set of the limited sensing scenario, we can calculate with the Hungarian Algorithm as the optimal upper bound, 
Fig.~\ref{fig:alg2_baseline}-(b) shows that  $q(\texttt{GREEDY})$ is close to $q(\texttt{HUNGARIAN})$ with $q(\texttt{GREEDY})/q(\texttt{HUNGARIAN)}\approx0.93$ for $M = \{1,\cdots,25\}$. In addition, $q(\texttt{GREEDY})$ is much higher than $\frac{1}{3}q(\texttt{HUNGARIAN})$. 
Thus, even though Algorithm~\ref{algorithm:limited_assignment} has a theoretical 1/3--approximation bound, it performs much better in practice. This is because the approximation ratio is computed by considering the worst-case performance of Algorithm~\ref{algorithm:limited_assignment}. 
% This high performance in practice comes from our worst-case assumption when deriving the $1/3$-approximation lower baseline for our algorithm.

\begin{figure}[t]
\centering{
\subfigure[Small scale]{\includegraphics[width=0.493\columnwidth]{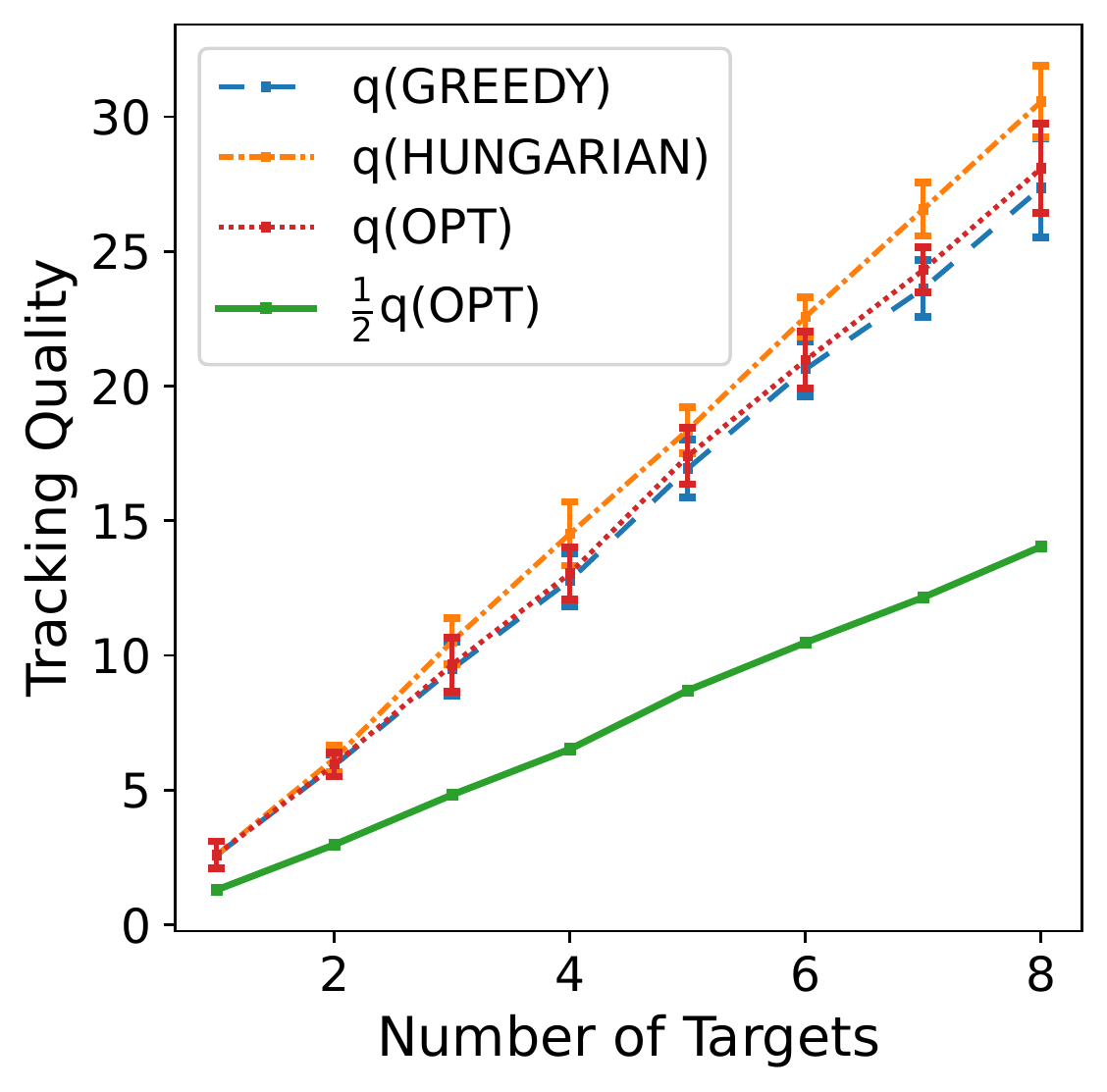}}
\subfigure[Large scale]{\includegraphics[width=0.493\columnwidth]{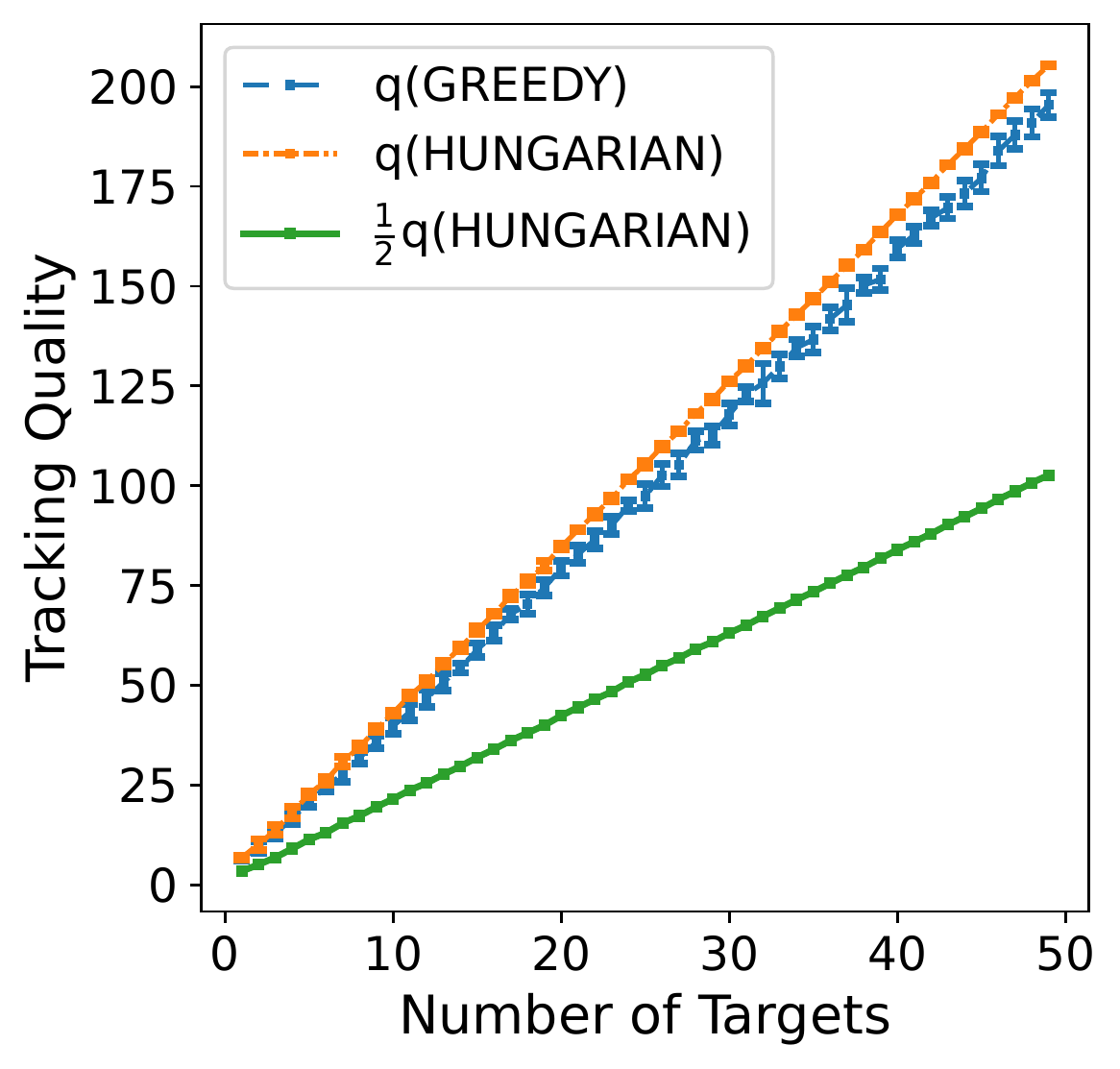}}
}

\caption{Comparison of the total tracking quality of Algorithm~\ref{algorithm:complete_assignment} with the exhaustive search algorithm in small-scale cases (a) and with the Hungarian algorithm in large-scale cases (b), respectively.} 
\label{fig:alg1_baseline}
\end{figure}

\begin{figure}[t]
\centering{
\subfigure[Small scale]{\includegraphics[width=0.493\columnwidth]{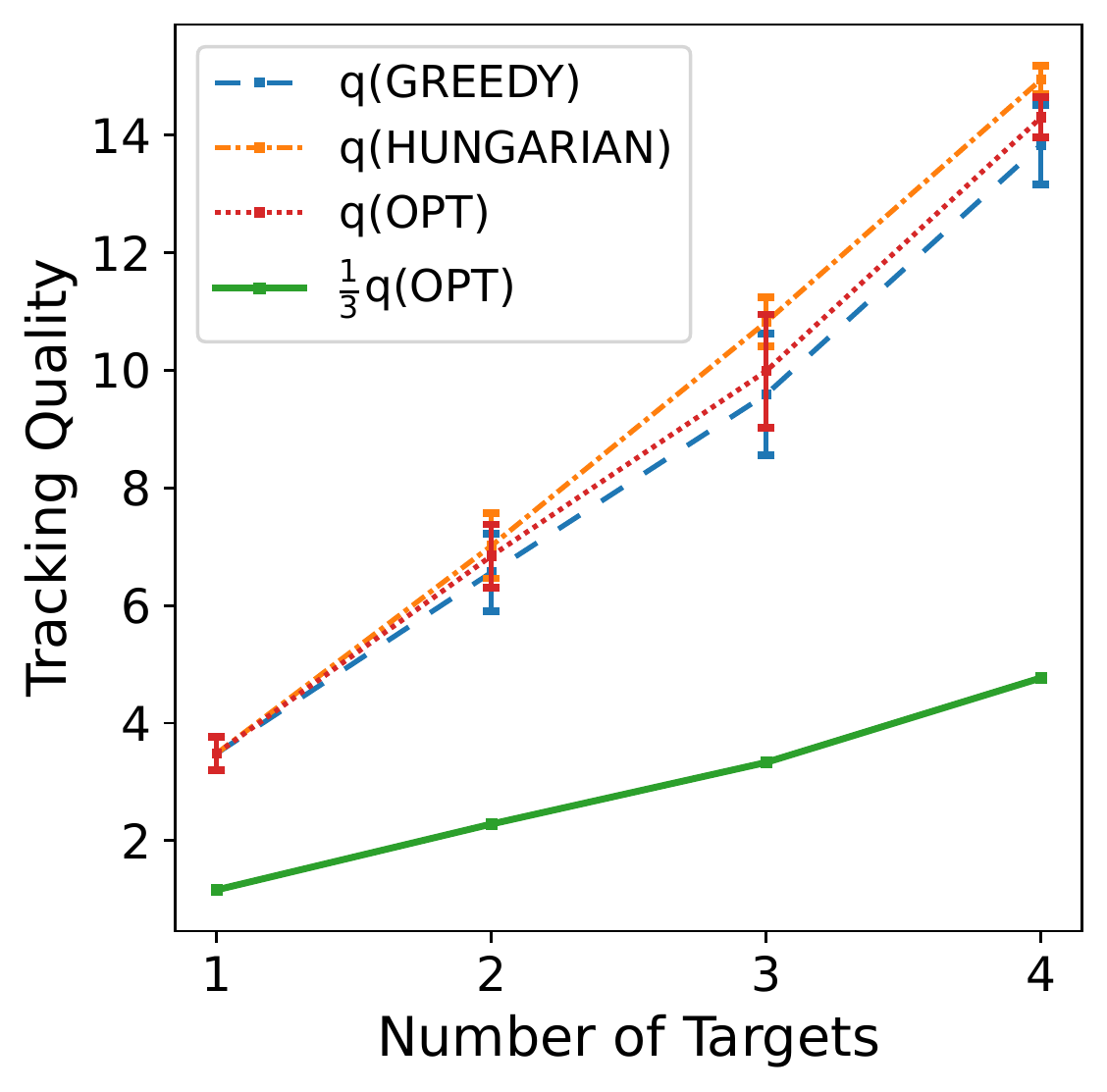}}
\subfigure[Large scale]{\includegraphics[width=0.493\columnwidth]{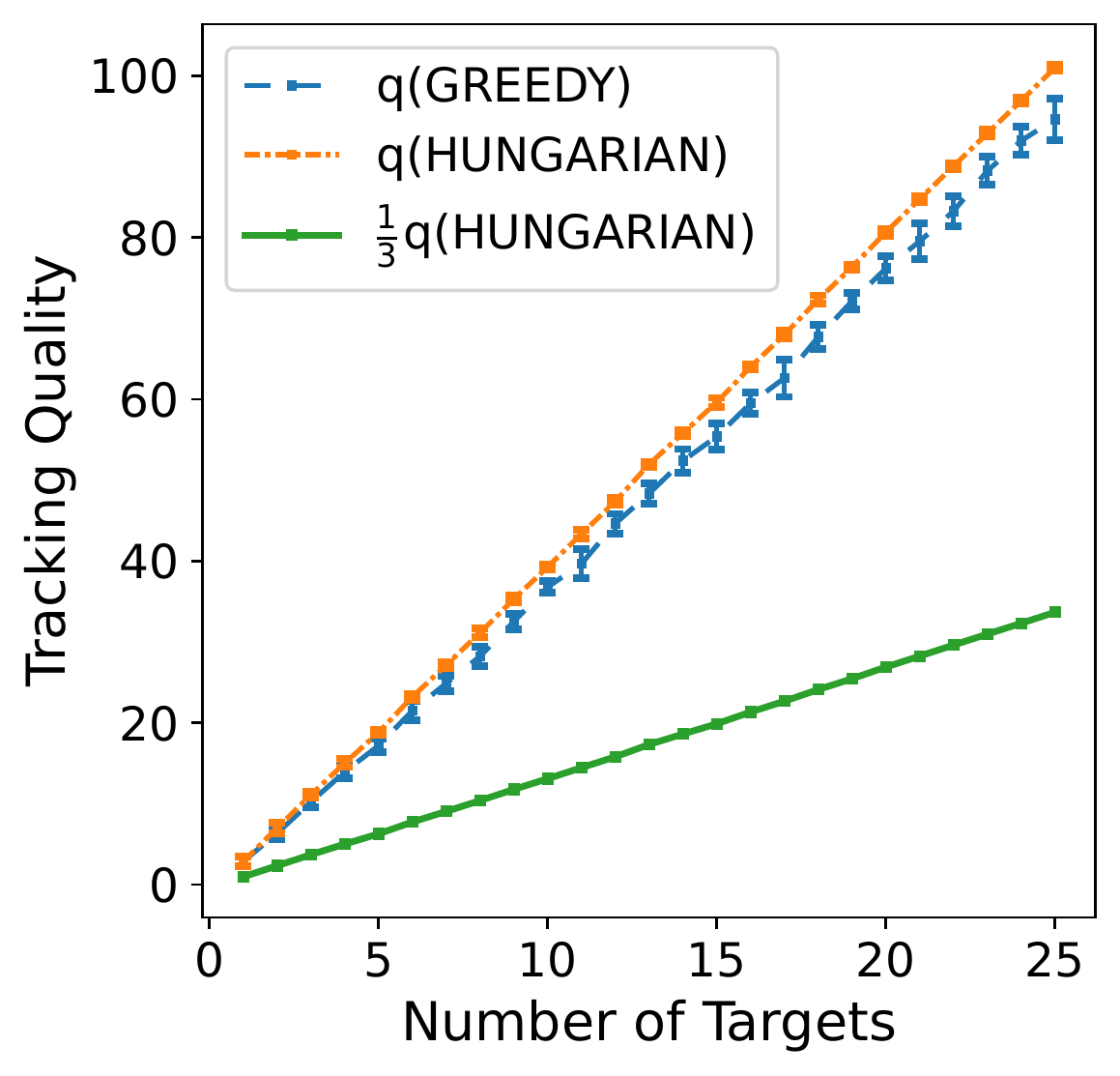}}
}

\caption{Comparison of the total tracking quality of Algorithm~\ref{algorithm:limited_assignment} with the exhaustive search algorithm in small-scale cases (a) and with the Hungarian algorithm in large-scale cases (b), respectively.} 
\label{fig:alg2_baseline}
\end{figure}

% \LZ{bearing sensors only}

\section{Conclusion}
We presented two greedy algorithms (Algorithm~\ref{algorithm:complete_assignment} and Algorithm~\ref{algorithm:limited_assignment}) to solve the problems of assigning robots with actions to track moving targets for robots with both sufficient and limited sensing capability. We derived and proved the approximation bounds for the two greedy algorithms, \textit{i.e.,} $1/2$--approximation bound for Algorithm~\ref{algorithm:complete_assignment} and $1/3$--approximation bound for Algorithm~\ref{algorithm:limited_assignment}. Moreover, we illustrated the algorithms' effectiveness through Gazebo simulations, where the assignment computed by the algorithms consistently steers robots toward following and tracking the targets. Further, we conducted quantitative comparisons to demonstrate that the algorithms perform close to the optimal solution and much better than the theoretical bound.

% A worst-case performance guarantee ($1/2$-approximation of optimal for sufficient sensing, and $1/3$-approximation of optimal for limited sensing) was presented and proved. 
% Then, the simulation in ROS-Gazebo showed the actions selected by the algorithms could consistently keep robots following and tracking the targets. A quantitative simulation was also presented showing the performance of the algorithm works much better than the theoretical bound in the application.

The assignment problems studied are centralized. Therefore, 
an immediate future research direction is to investigate the decentralized versions of the proposed problems and design corresponding decentralized assignment algorithms~\cite{qu2019distributed}. 
Another future direction is to design resilient assignment algorithms, which adaptively reconfigure the team and reassign robots when robot (or sensor or communication) failures occur ~\cite{ramachandran2020resilient,zhou2023robust}.

\bibliographystyle{IEEEtran}
\bibliography{refs}

\end{document}